\documentclass[11pt,DIV=12,BCOR=0.7cm]{article}   
\pdfoutput=1

\usepackage[left=2cm, right=2cm, top=2cm, bottom=3cm]{geometry}
\usepackage{graphicx,amssymb}
\usepackage{amsmath,amsfonts, amsthm}
\usepackage{bm}
\usepackage{mathtools,mathdots}
\usepackage{subcaption}
\usepackage{algorithm,algorithmic}
\usepackage{setspace}
\usepackage{xcolor}
\usepackage{scalerel,stackengine}
\usepackage{mathtools}
\usepackage{tikz}
\usetikzlibrary{matrix,chains,positioning,decorations.pathreplacing,arrows}
\usetikzlibrary{positioning,calc}
\usepackage{url}
\usepackage{wrapfig, multirow}
\usepackage{tablefootnote}
\usepackage{textcomp}
\usepackage[hidelinks]{hyperref}

\newtheorem{theorem}{Theorem}[section]

\newcommand*\diff{\mathop{}\!\mathrm{d}}
\DeclarePairedDelimiterX{\infdivx}[2]{\big(}{\big)}{%
	#1\;\delimsize\|\;#2%
}
\newcommand{\infdiv}{\text{KL}\infdivx}
\DeclarePairedDelimiter{\norm}\lVert\rVert

\DeclareMathOperator{\E}{\mathbb{E}}
\newcommand\equalhat{\mathrel{\stackon[1.5pt]{=}{\stretchto{%
				\scalerel*[\widthof{=}]{\wedge}{\rule{1ex}{3ex}}}{0.5ex}}}}
			
\newcommand{\expnumber}[2]{{#1}\mathrm{e}{#2}}

\usepackage[normalem]{ulem}

\newenvironment{keywords}
{\bgroup\leftskip 20pt\rightskip 20pt \small\noindent{\bf Keywords:} }%
{\par\egroup\vskip 0.25ex}

\makeatletter
\DeclareRobustCommand{\Udots}{%
	\vcenter{\offinterlineskip
		\halign{%
			\hbox to .8em{##}\cr
			\hfil.\cr\noalign{\kern.2ex}
			\hfil.\hfil\cr\noalign{\kern.2ex}
			.\hfil\cr}%
	}%
}
\makeatother

\title{Concurrent Training and Layer Pruning of Deep Neural Networks}
\author{Valentin Frank Ingmar Guenter\thanks{Graduate Student, Email:
vguenter@uci.edu}\ \ and\ \ Athanasios~Sideris\thanks{Professor, Email:
asideris@uci.edu}\\ \\  Department of Mechanical and
Aerospace Engineering,\\ University of California, Irvine,\\
Irvine, CA, 92697}
\begin{document}	
\maketitle
\begin{abstract}
We propose an algorithm capable of identifying and eliminating irrelevant
layers of a neural network during the early stages of training. In contrast to weight or filter-level pruning, layer pruning reduces the harder to parallelize sequential computation of a neural network. We employ a structure using linear residual connections around nonlinear network sections that allow the flow of information through the network once a nonlinear section is pruned. Our approach is based on variational inference principles using Gaussian scale mixture priors on the neural network weights and allows for substantial cost savings during both training and inference. More specifically, the variational posterior distribution of scalar Bernoulli random variables multiplying a whole layer weight matrix of its nonlinear sections is learned, similarly to adaptive layer-wise dropout. To overcome challenges of concurrent learning and pruning such as premature pruning and lack of robustness with respect to weight initialization or the size of the starting network, we adopt the ``flattening'' hyper-prior on the prior parameters.
We prove that, as a result of the usage of this hyper-prior, the solutions of the resulting optimization problem always describe deterministic networks with parameters of the posterior distribution at either 0 or 1.
We formulate a projected stochastic gradient descent algorithm and prove its convergence to such a solution using well-known stochastic approximation results.  In particular, we establish and prove analytically conditions that lead to a layer's weights converging to zero and derive  practical pruning conditions from the theoretical results.
The proposed algorithm is evaluated on the MNIST, CIFAR-10 and ImageNet datasets and the commonly used fully-connected, convolutional and residual architectures LeNet, VGG16 and ResNet. The simulations demonstrate that our method achieves state-of the-art performance for layer pruning at reduced computational cost in distinction to competing methods due to the concurrent training and pruning.
\end{abstract}


\section{Introduction}
Deep learning has gained tremendous prominence during recent years as it has been shown to achieve outstanding performance in a variety of machine learning tasks, such as natural language processing, object detection, semantic image segmentation and reinforcement learning \cite{girshick_fast_2015, noh_learning_2015, silver_mastering_2017}.
With the emergence of residual architectures such as, for example, ResNets \cite{ResNet_he_2016,he_identitymappings2016} and Batch Normalization \cite{batchnorm_ioffe2015}, Deep Neural Networks (DNNs) can be trained even when the number of layers are well into the hundreds.
The modern trend of making these architectures deeper can lead to popular architectures being unnecessarily over-parameterized for a task at hand, resulting in excessive computational requirements both during training and inference. As a result it is often difficult to train such architectures and/or infeasible to deploy them on systems with limited computational resources, e.g., low-powered mobile devices.

Neural network pruning \cite{han_learning_2015,blalock_what_2020} has been one technique used to reduce the size of over-parameterized Neural Networks (NNs) and focuses mainly on eliminating weights and/or nodes from the network while maintaining its prediction accuracy.
While unstructured weight pruning \cite{han_learning_2015} prunes individual weights from a network's layer, structured node pruning (e.g. \cite{li_pruning_2017, he_soft_2018, liebenwein_lost_2021, SCOP_tang_2020}) aims to remove entire neurons or filters from the network. In effect, the network is pruned to a narrower counterpart with the same depth. This allows accelerated inference with standard deep
learning libraries; in contrast, the practical acceleration of DNNs achieved with unstructured or weight
pruning may be limited by poor cache locality and jumping memory access caused by the ensuing random
connectivity of the network and require specialized hardware to perform best.
However, with the ability to heavily parallelize the computations within a layer of the network on graphics processing units (GPUs), the sequential workload of computing layer after layer contributes most considerably to the workload and time needed for training and inference in modern DNN architectures. This makes pruning complete layers from the network especially attractive since the non-parallelizable workload of training/evaluating the network is greatly reduced and given the same reduction in parameters of floating point operations (FLOPS), the network efficiency is potentially higher when compared to weight or node pruning.

The existing literature on layer pruning is limited, due to the modern trend to make DNNs deeper to achieve better accuracy results on more complicated tasks. However, these popular and well performing deeper architectures such as ResNets \cite{ResNet_he_2016, he_identitymappings2016} are often adopted for other tasks and are trained on a variety of datasets. Hence, the ability to effectively prune them, and especially in a layer-wise fashion, holds considerable potential. In the case of unstructured and structured node pruning, training larger networks and pruning them to smaller size often leads to superior performance compared to just training the smaller models from a fresh initialization \cite{blalock_what_2020, frankle_lottery_2019, Louizos_2017}. We find that the same holds true for layer pruning in our experiments.

\paragraph{Review of existing layer pruning studies and methods:}
In \cite{resnets_ensembles_veit16} residual networks are studied and it is first observed that with standard network training the resulting network often consists of a few most important layers and many layers of less importance. As a result when deleting such a single layer randomly from the network and propagating only over the skip-connection of the residual structure the error is not increased much. However, even if most layers can be removed in isolation from the network without affecting the error in a material way, removing many of them at the same time leads to a larger increase in error. The authors of \cite{resnets_ensembles_veit16} do not report pruning results with respect to test accuracy and pruning ratios.
To address the diminishing forward flow, the vanishing gradients and the high computational effort in very deep convolutional networks, \cite{stoch_depth_huang16} proposes \textit{stochastic depth} as a training procedure for residual networks using Bernoulli random variables (RVs) multiplying structures of the network; similar to how Dropout \cite{Srivastava_2014_Dropout} acts on the features of a layer, the RVs of stochastic depth act on the whole convolutional layer path of a residual block. Thus, training such a network can be thought of training an ensemble of networks which variable depth. During training and for each mini-batch a subset of layers is dropped and in effect a lot shorter networks are trained. At test time the full, deeper network is used. The authors demonstrate how stochastic depth is able to successfully train networks beyond 1200 layers while still improving performance and reducing training times at the same time. In their work, all distributions of RVs are chosen beforehand and kept constant during training, i.e. the dropout rates are constant.

Many existing pruning methods are applicable only on pre-trained networks and, therefore, require the costly initial training of the often large models. In this group, \cite{chen19_shallowing} identifies and removes layers from feed-forward and residual DNNs by analyzing the feature representations computed at different layers. A set of linear classifiers on features extracted at different intermediate layers are trained to generate predictions from these features and identify the layers that provide minor contributions to the network performance; the loss in performance when removing a layer is compensated via retraining and a knowledge distillation technique.
In \cite{DBPwang2019}, \textit{discrimination based block-level pruning} (DBP) is introduced to remove redundant layer blocks (a sequence of consecutive layers) of a network according to the discrimination of their output features. To judge the discriminative ability of each layer block in an already trained network, again additional fully-connected layers are introduced after each block as linear classifiers. The added layers are trained and their accuracy is used as a measure of how discriminative the output features of each block in the network are. Blocks with the least discriminative output features are pruned from the network. This method requires additional training of the network's weights after pruning. In \cite{LPSR_zhang},
\textit{Layer Pruning for Obtaining Shallower ResNets} (LPSR) aims to prune the least important residual blocks from the network to achieve a desired pruning ratio. Each such block is followed by a Batch Normalization (BN) layer \cite{batchnorm_ioffe2015} and can be effectively removed by setting every factor of the BN layer to zero. The impact to the loss function is approximated using first order gradients according to \cite{molchanov_pruning_conv} and based on these gradients, an importance score is assigned for each residual block and blocks with the lowest score are pruned. We note that this method considers disconnecting one convolutional block at a time only and it is unclear how removing two or more blocks at the same time would affect the loss. A follow-up fine-tuning procedure is needed to obtain best results.

The Lottery Ticket Hypothesis \cite{frankle_lottery_2019} states that there exist winning tickets given by subnetworks of a larger NN that can achieve on par performance after training. However, identifying these winning tickets is not trivial and hence most existing pruning methods operate on pre-trained large networks. This may not be a viable choice if the computational resources required to train very deep networks are not available and transfer learning is not applicable. Concurrent training/pruning of deep neural networks aims to reduce the computational and memory load during training as well as prediction.
Carrying out the pruning process simultaneously with learning the network's weights imposes a major challenge. It is desirable that pruning is effected as early as possible during training to maximize the computational efficiency, but at the same time the process
needs to be robust with respect to the initialization of the weights and the starting architecture. This means that even if a weight or structure of the network is initialized unfavorably, it is given a chance to adapt and learn useful behavior instead of getting pruned prematurely due to its poor initialization.

Existing methods that in principle carry out the pruning simultaneously to the network training for layer pruning are summarized in the following.
In \cite{sparse_huang2018}, scaling factors are introduced for the outputs of specific structures, such as neurons, group or residual blocks of the network. By adding sparsity regularization in form of the $\mathcal{L}_1$-norm on these factors and in the spirit of the Accelerated Proximal Gradient (APG) method of \cite{proximal_algs_Parikh2014}, some of them are forced to zero during training and the corresponding structure is removed from the network leading to a training and pruning scheme procedure.
However, \cite{sparse_huang2018} does not consider any computational savings during training and does not reported at what stage during training most of the pruning process is completed.
In \cite{xu2020layer}, each convolutional layer is replaced  by their \textit{ResConv} structure, which adds a bypass path around the convolutional layer, and also scaling factors to both paths. The scaling factors are learned during training and $\mathcal{L}_1$-norm regularization is placed on them to force a sparse structure. Convolutional layers with a small layer scaling factor are only removed after training, and hence not saving computational load during training in \cite{xu2020layer}. The authors note that retraining the network is needed when the pruning rate is high.

Bayesian variational techniques together with sparsity promoting
priors have been also employed for neural network pruning \cite{Louizos_2017,molchanov2017variational, Nalisnick_2015}. Such methods typically employ Gaussian scale mixture priors, which are zero mean normal probability density functions (PDF’s) with variance (scale) given by another RV. A notable prior in this class is the spike-and-slab prior, in which only two scales are used.
Dropout \cite{Srivastava_2014_Dropout}, with its Bernoulli distribution, has been interpreted as imposing such a spike-and-slab PDF on the weights of a NN \cite{gal2015dropout, Louizos_2017}.
In \cite{Nalisnick_2015}, for the purpose of weight pruning, an identifiable parametrization of the multiplicative noise is used where the RVs are the product of NN weights and the scale variables themselves. Then, estimates of the scale variables are obtained using the Expectation-Maximization algorithm to maximize a lower bound on the log-likelihood. The expectation step is accomplished using samples from the posterior of the NN weights via Monte Carlo (MC) simulations. Weights with a sum of posterior variance and mean less than a set threshold are pruned.
In \cite{molchanov2017variational}, multiplicative Gaussian noise
is also used on the network weights, which receive an improper logscale uniform distribution as prior. It postulates a normal posterior on the weights and proceeds to maximize the Evidence Lower Bound (ELBO) over its parameters. Due to the choice of prior, an approximation to the Kullback–Leibler divergence term of the ELBO is necessary. The authors discuss that their method can be sensitive to weight initialization and extra steps must be taken to assure good initialization.

In \cite{pmlr-v97-nalisnick19a}, based on variational inference, \textit{automatic depth determination} for residual neural networks is introduced. Falling into the same category of variational methods, it aims to select layers and hence the appropriate depth of fully connected NNs.
Different marginal priors on the networks weights are considered, one of which being the spike-and-slab prior with two scales. In this case, multiplicative Bernoulli noise is used in each residual block, revealing equivalences with \textit{stochastic depth} \cite{stoch_depth_huang16}, but with adaptive dropout rate.
This way shrinkage of the corresponding weights to zero and therefore pruning of the layer is effected if a scale parameter converges to zero.
While \cite{pmlr-v97-nalisnick19a} reports the loss values for their method on different regression datasets, no explicit pruning results are stated and it is unclear by how much the networks are pruned. Computational load during training or prediction is not considered.

\paragraph{Overview of the Proposed Algorithm}
In this work, we propose a concurrent training and pruning algorithm capable of identifying redundant or irrelevant layers in a neural network during the early stages of the training process. Thus, our algorithm is capable of reducing the computational complexity of subsequent training iterations in addition to the complexity during inference. We use multiplicative Bernoulli noise, i.e., we propose a spike-and-slab method and use variational techniques; therefore, our method falls in the same category as \cite{Nalisnick_2015, pmlr-v97-nalisnick19a}.
This can be interpreted as adding layer-wise dropout \cite{stoch_depth_huang16} to the network, however, each layer receives its own adaptive dropout rate. During training via backpropagation, different subnetworks formed from the active layers are realized. The probabilities determining the active layers are learned using variational Bayesian principles.
To address the challenges of simultaneous training and pruning, we follow the ideas of \cite{GUENTER2024_robust} and adopt the ``flattening'' hyper-prior over the parameters of the Bernoulli prior on the scale RVs.
In addition, we carefully analyze the dynamical system describing the learning algorithm to gain insight and develop pruning conditions addressing the challenge of premature pruning.
There are several key contributions in this work, distinguishing it from previous literature:
\begin{enumerate}	
	\item Since we use multiplicative Bernoulli noise with adaptive parameters, it is important to ensure they are learned between $0$ and $1$.
	In this work and in contrast to \cite{GUENTER2024_robust}, we make use of projected gradient descent to learn these parameters. This leads to a cleaner overall problem formulation  in Section~\ref{sec:Problem_Formulation} and circumvents some of the technical difficulties in the problem formulation of \cite{GUENTER2024_robust}.
	
	\item We analyze the resulting constrained optimization problem over the network weights and variational Bernoulli parameters in Section~\ref{sec:Main-Results} carefully and show that its solutions are always deterministic networks, i.e., the Bernoulli parameters are either $0$ or $1$ in each layer of the network. This property is a further consequence of adopting the ``flattening'' hyper-prior and strengthens its importance for obtaining well regularized and efficiently pruned NNs. In addition, our analysis gives new insight into how the choice of the hyper-prior parameter measures the usefulness of a surviving layer. We note that these results can be extended to the problem formulation in \cite{GUENTER2024_robust} using the flattening hyper-prior for unit pruning and confirm the empirical observation made there of consistently arriving at deterministic pruned networks.
	
	\item We present a stochastic gradient descent algorithm for finding the optimal solution in Section~\ref{sec:algorithm} and analyze its convergence properties in Section~\ref{sec:convergence_results}. Stochastic approximation results are used to prove convergence of our algorithm to an optimal solution using dynamical systems stability results; we establish a region of attraction around $0$ for the dynamics of the parameters of the posterior PDFs on the scale RVs and the NN weights and provide provable conditions under which layers that converge to their elimination cannot recover and can be safely removed during training. This theoretical result leads to practical pruning conditions implemented in our algorithm and addresses the challenges of robust and concurrently training and pruning.

	\item Our method does not require much more computation per training iteration than standard backpropagation and achieves state-of-the-art layer pruning results. In fact, because of the discrete variational posterior, only some layers of the network are active during training, leading to additional computation savings \cite{graham2015efficient}. The simultaneous removal of layers of the network during training considerably reduces training times and/or expended energy. In our experiments, this reduction is of the order of $2$ to $3$-fold when training/pruning the ResNet110 architecture \cite{ResNet_he_2016, he_identitymappings2016} on the CIFAR-100 \cite{cifar10} dataset while only a minor loss in prediction accuracy occurs.
\end{enumerate}

The remainder of this paper is arranged as follows. In Section~\ref{sec:Problem_Formulation}, we present the statistical modeling that forms the basis of our layer pruning approach and formulate the resulting optimization problem. Section~\ref{sec:Main-Results} provides the analysis of the optimization problem and characterization of its solutions. In Section~\ref{sec:algorithm}, we summarize  the proposed concurrent learning/pruning algorithm. Convergence results supporting the pruning process are provided in Section~\ref{sec:convergence_results} and~\ref{app:assumption_verification}. In Section~\ref{sec:simulations}, we present simulations on standard machine learning problems and comparison with state-of-the-art layer pruning approaches. Finally, Section~\ref{sec:conclusions} concludes the paper.
\paragraph{Notation:}
We use the superscript $l$ to distinguish parameters or variables of the $l$th layer of a neural network  and subscript $i$ to denote dependence on the $i$th sample in the given dataset. We use $n$ in indexing such as $x(n)$ to denote iteration count. Also $\norm{\cdot}$ denotes the Euclidean norm of a vector, $M^T$ transpose of a matrix (or vector), and $\E[\cdot]$ taking expectation with respect to the indicated random variables. Other notations are introduced before their use.

\section{Problem Formulation}\label{sec:Problem_Formulation}
\subsection{Network Architectures for Layer Pruning}\label{sec:net_arch}
Layer pruning in feedforward NNs requires the use of a suitable network architecture. More specifically, once a layer (or a block of consecutive layers) is removed, the flow of information is interrupted and the network collapses unless a bypass path around the pruned layer (or block) is present. Such a bypass path should then be possible to absorb either in the preceding or following layer and thus completely eliminate its computational load; or at least it should have relatively low computational/memory requirements, if it has to remain in the place of the pruned layer. Residual NNs, such as ResNets \cite{ResNet_he_2016, he_identitymappings2016}, include either identity skip-connections or low-cost convolutional skip-connections that form shortcuts over multiple of the more expensive layer operations; these connections can serve the role of the required bypass paths. In purely sequential networks, such as for example VGG16 \cite{simonyan_deep_2015}, skip connections need to be introduced around the sections of the network that are candidates for pruning.

A section of a general NN structure that allows a unified approach for layer pruning for most of the different types of networks proposed in the literature is depicted in Fig.~\ref{fig:NN_struct_general}.
In this section, the part of the network that is a candidate for pruning is denoted by $Block^l(W,a)$ and may represent a single or multiple consecutive layers of weights $W$ and activation functions $a$. The indicated skip connection may be already a part of the NN as in the ResNet structures or supplied as a necessary part of the pruning process. A central part of our approach is the use of discrete Bernoulli $0/1$ random variables (RVs) $\xi^l\sim Bernoulli(\pi^l)$ with parameter $\pi$ as shown in Fig.~\ref{fig:NN_struct_general}. These RVs are sampled during training in a dropout fashion and their parameters $\pi^l$ are adjusted as explained in the following sections so that they converge either to $0$ or $1$ signifying the pruning or retaining of the $Block^l(W,a)$ section, respectively. Finally, we include functions $h^l(\cdot)$ to provide the generality needed to encompass the various NN architectures available in the literature in our framework; for example, $h^l(\cdot)$ can be the identity or an activation function or a pooling/padding operation.

Table \ref{tab:architectures} shows how existing network structures such as ResNets \cite{ResNet_he_2016, he_identitymappings2016} and VGG networks \cite{simonyan_deep_2015} can be cast in the structure of Fig.~\ref{fig:NN_struct_general}. $W$ denotes weights in fully-connected or filters in convolutional layers and ``$a$'' an activation function. $B$ denotes a batch normalization layer in the architecture and $I$ is the identity, i.e. no activation function is used. ResNet \cite{ResNet_he_2016} uses residual layer blocks of the structure $W$-$B$-$a$-$W$-$B$, where each weight $W$ represents one convolution in the network followed by a ReLU (Rectified Linear Unit) activation function. The full pre-activation structure of ResNet, named ResNet-v2, uses blocks of the form $B$-$a$-$W$-$B$-$a$-$W$ and no activation function after the skip connection. In both cases, the operation in the skip connection is either identity, a padding or pooling operation or a 1x1 convolution to match the dimensionality of the block output $\bar z$. To frame the VGG architecture in our general structure, we use blocks of the form $W$-($B$), where the batch normalization is optional, we take the functions $h(\cdot)$ to be ReLUs and disconnect the skip connection, signifying that there is no skip connection in the original VGG network; thus, to implement our pruning process on the VGG network, the network is modified by considering a skip connection. Finally, in all three of the above structures the RVs $\xi$ are not present and are added as indicated previously as a part of our pruning method.

\begin{figure}[t!]
    \centering
    \includegraphics[width=0.7\textwidth]{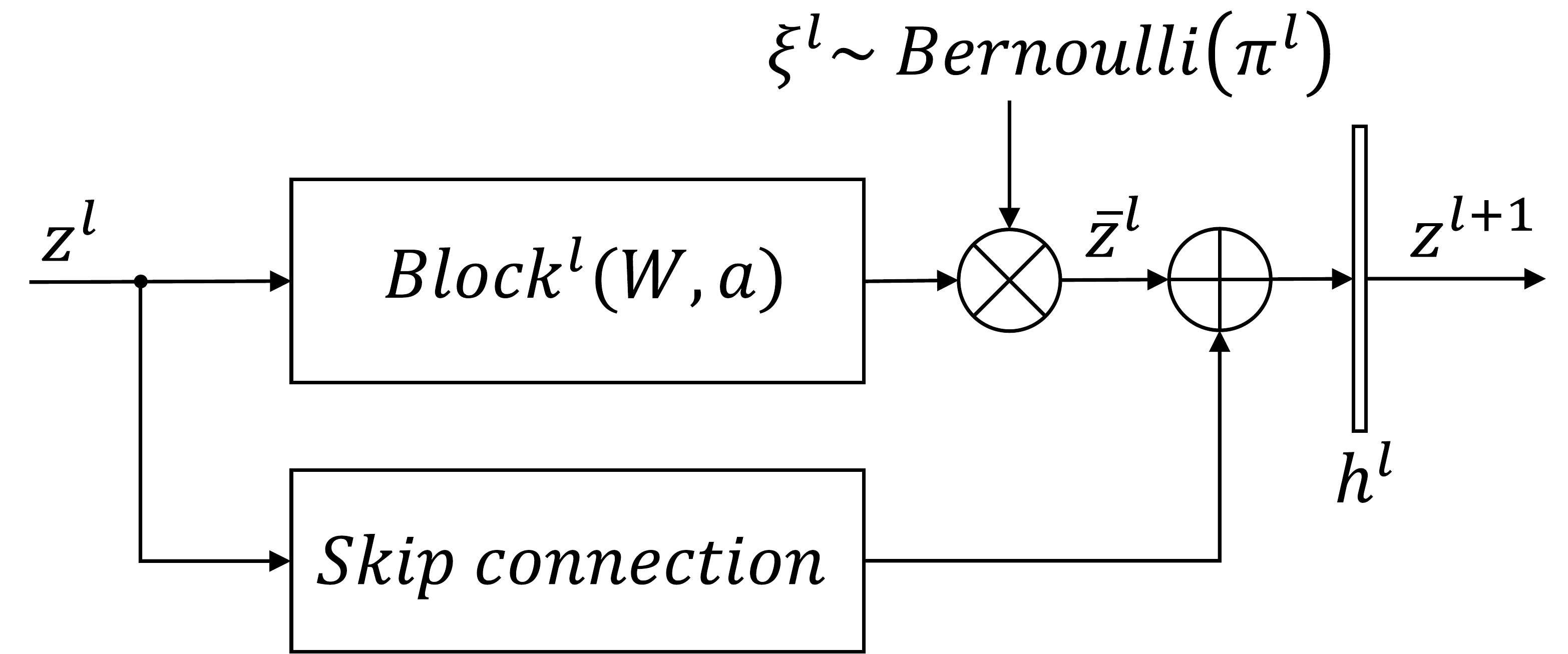}		
	\caption{A section of the general NN structure considered in this work. $z^l$ and $z^{l+1}$ are feature vectors in the NN. $Block^l(W,a)$ represents a single or multiple consecutive layers of weights $W$ and activation functions ``$a$''. $\xi^l$ are Bernoulli Random Variables with parameter $\pi^l$ and $h^l$ is an activation function.}\label{fig:NN_struct_general}
\end{figure}

\begin{table}
	\centering
	\caption{Different Network Architectures obtained from the general network structure in Fig.~\ref{fig:NN_struct_general}. All super- and subscripts are omitted. $W$ denotes a fully connected or convolutional layer, $B$ batch normalization layers and $a$ the activation functions.}
	\label{tab:architectures}
	\begin{tabular}{lllll}
		&  & Nonlinear path ($Block$)  & $h$ & Skip connection  \\
		\hline
		&ResNet  & $W-B-a-W-B$ & $a$ & Identity, padding/pooling, 1x1 convolution  \\
		&ResNet-v2  & $B-a-W-B-a-W$  & $I$ & Identity, padding/pooling, 1x1 convolution  \\
		&VGG  & $W-(B)$ & $a$  & $0$\\
		&Ours & $W-(B)-a-W$ & $I$ & Identity, padding/pooling, 1x1 convolution	
	\end{tabular}
\end{table}

In addition to the existing architectures,  we add the architecture labeled ``Ours'' in Table~\ref{tab:architectures} used to develop our results and depicted separately in Fig.~\ref{fig:NN_struct_specific}. However, these results are readily extended to any architecture in which ``Block'' ends with a $W$ component, as all of the shown architectures in Table~\ref{tab:architectures} do. A consideration in selecting our architecture, besides its simplicity, is that its block is a one-hidden layer network that by the universal approximation theorem \cite{HORNIK1989} can be made arbitrarily expressive by increasing the width of its hidden layer. In this manner, its function can be formed independently from adjacent layers and the block can be more easily removed if not essential to the overall function of the network. Furthermore, once a block in our architecture is pruned, the bypass connection can be easily absorbed within the remaining sections as explained in detail later in eq. \eqref{eq:lin_layer_reduction}.

We focus on the proposed residual structure (labeled ``Ours'' in Table~\ref{tab:architectures}) with fully connected layers for the remainder of the theoretical sections in this work. This structure yields a network with $L-1$ residual blocks realizing mappings $\hat y = NN(x;W,\Xi)$ as follows:
\begin{align}\label{eq:struct_res}
\begin{split}
z^1 &=x;\\
\bar z^l &=\xi^l\cdot \left(W^l_2 a^l\left(W^l_1 z^l+b_1^l\right) + b_2^l\right);\quad
z^{l+1} =  \bar z^l + W_3^l z^l \quad \quad\text{for}\quad l=1 \dots L-1,\\
\bar z^{L} &= W^Lz^L+b^L;\quad z^{L+1}=a^L(\bar z^L)\\
\hat y &= z^{L+1}.
\end{split}
\end{align}

\begin{figure}[t!]
	\centering
	\includegraphics[width=0.7\textwidth]{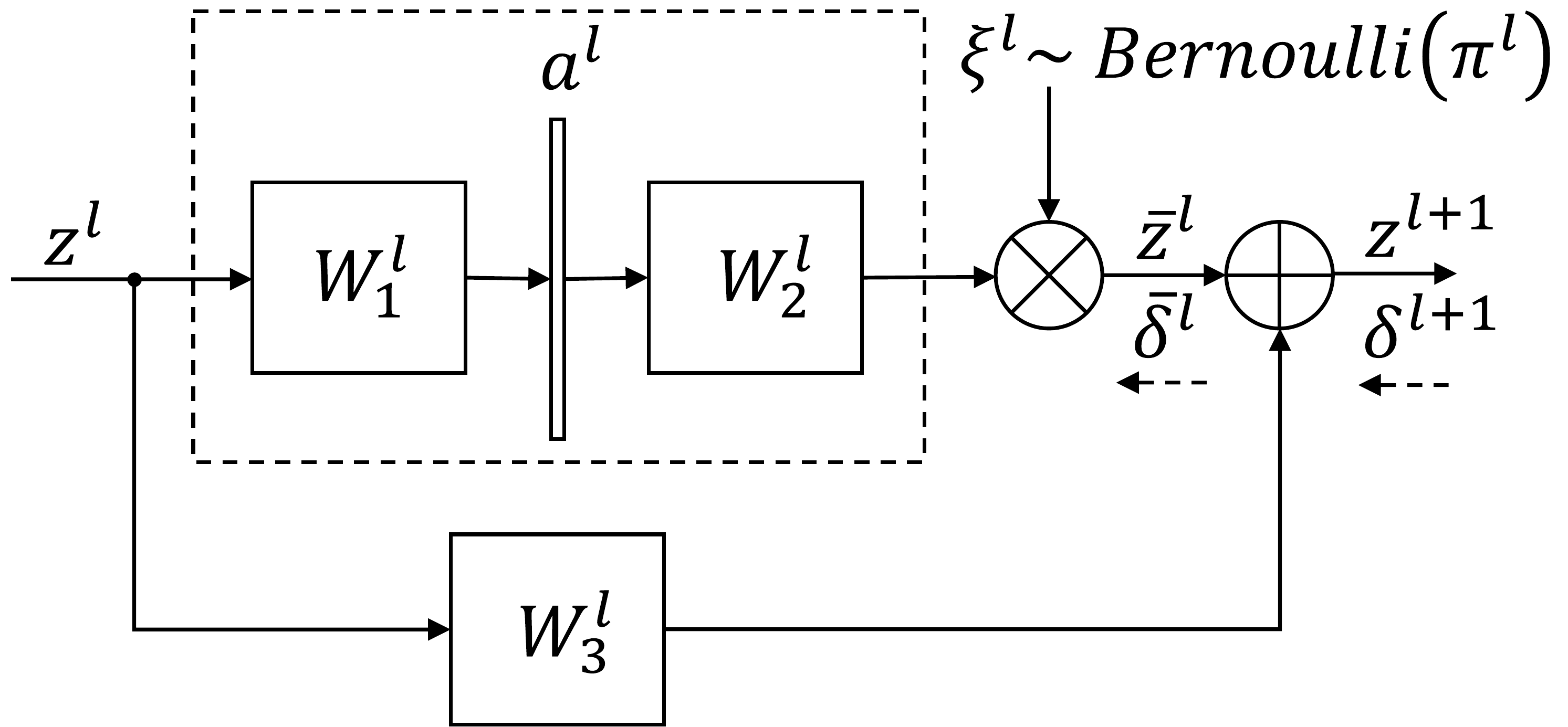}		
	\caption{The specific network structure considered in this work (Ours res. in Table \ref{tab:architectures}). $z^l$ are input features, $a^l$ the activation function and $W_1^l, W_2^l$ and $W_3^l$ network weights. The Bernoulli RV $\xi^l$ with parameter $\pi^l$ multiplies the output of the block $W_1^l-a^l-W_2^l$ and yields $\bar z^l$; $\bar \delta^l$  and $\delta^{l+1}$ are back-propagated gradients of the loss function with respect to $\bar z^l$ and $z^l$, respectively.}\label{fig:NN_struct_specific}
\end{figure}

In this model,  $\xi^l\sim Bernoulli(\pi^l)$ are scalar discrete Bernoulli $0/1$ random variables with parameter $\pi^l$.  The matrices $W^L,W_1^l, W_2^l, W_3^l, b^L, b_1^l, b_2^l$ are of appropriate dimensions and the element-wise activation functions $a^l(\cdot)$ are assumed to be continuously differentiable and satisfy $a^l(0)=0$ for $l=1\dots L-1$. To simplify the notation, we absorb the additive bias vectors $b^L$, $b^l_1$ and $b^l_2$, usually used in neural network architectures into the weight matrices$W^L$, $W^l_1$ and $W^l_2$, respectively as an additional last column, while adding a last row of zeros of appropriate dimension. We also extend $z^l$  to $z^l = \begin{bmatrix} (z^l)^T & 1 \end{bmatrix}^\top$ and correspondingly $W_3^l$ by adding a last column of zeros and a last row $\begin{bmatrix} 0 &\dots& 0 & 1 \end{bmatrix}$ of appropriate dimension.  The activation on the extended elements of $z^l$ is assumed to be the identity.

We denote collectively $W^L$ and $W_1^l, W_2^l, W_3^l$, $l=1,\ldots,L$ by $W$ and  $\xi^l$, $\pi^l$, $l=1,\ldots,L-1$ by $\Xi$ and $\Pi$, respectively.
We will refer to the calculation of the term $\bar z^l$ in \eqref{eq:struct_res} as the nonlinear path and the calculation of $W_3^lz^{l}$ as the linear or residual path.
Realizations of the Bernoulli $\xi^l$ random variables implement sub-networks of structure~\eqref{eq:struct_res}, where the $l$th nonlinear path is active with probability $\pi^l$. Specifically, if $\xi^{l-1}= 0$ in layer $l-1$, its nonlinear path is always inactive and its output is given by the linear path only. Then, $\bar z^{l-1}=0$ and
\begin{align}\label{eq:lin_layer_reduction}
\begin{split}
	z^l &=   \bar z^{l-1} + W_3^{l-1} z^{l-1} = W_3^{l-1} z^{l-1} \\
	z^{l+1} &= \underbrace{\xi^l W^l_2 a^l\left(W^l_1 z^l\right) }_{=\bar z^l} + W_3^l z^l =\xi^l W^l_2 a^l\left(W^l_1  W_3^{l-1} z^{l-1}\right)  + W_3^l W_3^{l-1} z^{l-1},
\end{split}
\end{align}
and by replacing $W_1^l \leftarrow W_1^lW_3^{l-1}$ and $W_3^l \leftarrow W_3^lW_3^{l-1}$, layer $l-1$ is eliminated from the neural network.

\subsection{Statistical Modeling}\label{sec:statistical_model}
Given a dataset $\mathcal{D}=\lbrace (x_i,y_i)\rbrace_{i=1}^N$, where $x_i \in \mathbb{R}^m$ are input patterns and $y_i \in \mathbb{R}^n$ are the corresponding target values, the goal is to learn all three sets of weights $W$ and appropriate parameters $\Pi=\{\pi^1,\ldots,\pi^{L-1}\}$ for the prior distributions of the RVs $\Xi$ for the NN in \eqref{eq:struct_res}.

For regression tasks, the samples $(x_i,y_i)$ are assumed to be  drawn independently from the Gaussian statistical model
\begin{align}\label{eq:model_regression}
p(Y\mid X,W,\Xi) \sim \prod_{i=1}^N \mathcal{N}\left(y_i;NN(x_i,W,\Xi),\frac{1}{\tau}\right)
\end{align}
with a variance hyper-parameter $\frac{1}{\tau}$. In the following, we assume $\tau=1$ without loss of generality.
 For a $K$-class classification problem, we assume the categorical distribution and write the statistical model as
\begin{align}\label{eq:model_classification}
p(Y\mid X,W,\Xi) \sim \prod_{i=1}^N \prod_{k=1}^{K} (\hat y_{i,k})^{y_{i,k}},
\end{align}
where $y_{i,k}$ and $\hat y_{i,k}$ are the $k$th components of the one-hot coded target and NN output vectors, respectively for the $i$th sample point. We note that $X,Y$ denote collectively the given patterns and targets, respectively and not the underlying RVs. Therefore, \eqref{eq:model_regression} and \eqref{eq:model_classification} give the conditional likelihood of the targets when the NN model is specified.

We have already assumed Bernoulli prior distributions $p(\xi^l\mid \pi^l)$ for the RVs $\xi^l$'s in $\Xi$.
The weights in $W$ receive Normal prior distributions $p(W^l\mid \lambda)$.
Then, the overall prior distribution factorizes as follows
\begin{align*}
p(W,\Xi\mid \Pi, \lambda) = p(W\mid \lambda)\cdot p(\Xi \mid \Pi) = \prod_{l=1}^L p(W^l\mid \lambda) \cdot p(\xi^l\mid \pi^l)
\end{align*}
with
\begin{align*}
\begin{split}
p(W^l\mid \lambda) &\sim \mathcal N \big(0, \lambda^{-1} \mathbf{I}\big)\\
p(\xi^l \mid \pi^l ) &\sim Bernoulli\big(\pi^l\big) \propto ({\pi^l})^{\xi^l} (1-\pi^l)^{1-\xi^l}.
\end{split}
\end{align*}
Next, we place hyper-priors $p(\pi^l \mid \gamma)$ on the variables $\pi^l$ and by
combining the NN's statistical model with the prior distributions and integrating out $\Xi$,
we obtain the posterior
\begin{align*}
p(W,\Pi \mid Y,X,\lambda,\gamma)
&\propto p(Y\mid X,W,\Pi)\cdot p(W\mid\lambda)\cdot p(\Pi\mid\gamma).
\end{align*}

Maximum a Posteriori (MAP) estimation selects the parameters $\Pi$ and $W$  by maximizing the $\log$-posterior:
\begin{align}\label{eq:obj_pi}
\max_{W, \Pi} \left[\log p(Y\mid X,W,\Pi) +\log p(W\mid \lambda)+ \log p(\Pi\mid \gamma) \right].
\end{align}
We select the hyper-prior $p(\Pi\mid\gamma)$ to encourage the automatic removal of nonlinear layer paths that do not sufficiently contribute to the performance of the network during training.
Specifically, the choice of $p(\Pi\mid\gamma)$  affects the resulting values of $\pi^l$ and/or $W_2^l$ during the optimization in (\ref{eq:obj_pi}), the convergence of which to zero signals the removal of the $l$th nonlinear layer path from the NN. This in turn allows the absorption of its linear path to the next layer as shown in \eqref{eq:lin_layer_reduction} and thus the complete pruning of the $l$th layer.
In \cite{GUENTER2024_robust} and for the purpose of unit pruning, Bernoulli priors on vector valued RVs multiplying element-wise the features/filters of DNNs were used and the {\it flattening hyper-prior} was introduced and shown to be instrumental for the success of the resulting simultaneously training/unit-pruning algorithm. Here, we also use a flattening hyper-prior defined as the product over the layers of the network of terms
\begin{align*}
p(\pi \mid \gamma) = \frac{\gamma-1}{\log\gamma}\cdot\frac{1}{1+(\gamma-1)(1-\pi)}, \quad 0<\gamma<1
\end{align*}
for $\pi=\pi^l$, $l=1,\ldots, L-1$. The flattening hyper-prior leads to a robust pruning/learning method by providing a regularization effect that discourages
the premature pruning/survival of network substructures. In \cite{GUENTER2024_robust}, such substructures are individual units, while in this work the structures considered to prune are layers.

\subsection{Model Fitting via a Variational Approximation Approach}\label{sec:Variational_Approach}
For deep neural networks the exact $p(Y \mid X,W,\Pi)$ in  \eqref{eq:obj_pi} resulting from integrating out $\Xi$ is intractable, necessitating the approximation of the posterior distribution of $\Xi$. To this end, we employ variational methods and introduce the variational posterior
\begin{align*}
q(\Xi \mid \Theta) \sim Bernoulli(\Theta),
\end{align*}
and recall the Variational lower bound or {\it ELBO} (see e.g. \cite[Chapter 10]{bishop_pattern_2006})
\begin{align*}
\log p(Y\mid X,W,\Pi)
\geq \int q(\Xi\mid \Theta) \log p(Y\mid X,W,\Xi) \diff \Xi - \infdiv{q(\Xi\mid \Theta)}{p(\Xi\mid \Pi)} \equiv L_B(\Theta,W,\Pi),
\end{align*}
where $KL(\cdot\mid\cdot)$ denotes the Kullback-Leibler divergence.
Then, we proceed by replacing the intractable evidence $\log p(Y\mid X,W,\Pi)$ in \eqref{eq:obj_pi} with the ELBO $L_B(\Theta,W,\Pi)$ to obtain the maximization objective:
\begin{align}\label{eq:Objective_with_KL}
\begin{split}
\max_{W,\Pi,\Theta} \int q(\Xi\mid \Theta) \log p(Y\mid X,W,\Xi) \diff \Xi  - \infdiv{q(\Xi \mid \Theta)}{p(\Xi\mid \Pi)}  + \log p(W\mid \lambda)  + \log p(\Pi\mid \gamma).
\end{split}
\end{align}
The integral term in (\ref{eq:Objective_with_KL}) represents an estimate of the loss over the given samples; its maximization leads to parameters $W$ and $\Theta$ that explain the given dataset best by placing all probability mass of $q(\Xi\mid \Theta)$ where $p(Y\mid X,W,\Xi)$ is highest. Maximizing the second part in (\ref{eq:Objective_with_KL}) or equivalently minimizing the Kullback-Leibler divergence between the variational posterior and the prior distributions on $\Xi$  keeps the approximating distribution close to our prior. Finally, maximizing the last two terms in (\ref{eq:Objective_with_KL}) serves to reduce the complexity of the network by driving the update probabilities of nonessential layer paths and corresponding weights to zero.

Next, in the spirit of \cite{corduneanu_2001_model_sel} and \cite{GUENTER2024_robust}, we first define an optimization problem that can be solved explicitly for the parameters $\Pi$ in terms of $\Theta$ once the prior $p(\Pi\mid\gamma)$ has been specified. This step amounts to a type~II MAP estimation.
Specifically, to maximize the main objective \eqref{eq:Objective_with_KL} with respect to $\Pi$, we can equivalently minimize
\begin{align}\label{eq:maxobj_pi}
&\min_\Pi J_{tot}(\Pi)\equiv \infdiv{q(\Xi\mid \Theta)}{p(\Xi\mid \Pi)} - \log p(\Pi\mid \gamma).
\end{align}
 Noting that $J_{tot}(\Pi)$ in \eqref{eq:maxobj_pi} factorizes over $\pi^l$, $\theta^l$ we can express $J_{tot}(\Pi)$ for the flattening hyper-prior as the sum over the layers of the network of terms
\begin{align*}
	J(\pi) = (1-\theta) \log \frac{1-\theta}{1-\pi} + \theta \log\frac{\theta}{\pi} + \log \left(1+(\gamma-1)(1-\pi)\right).
\end{align*}
for $\pi=\pi^l$ and $\theta=\theta^l$, $l=1,\ldots,L-1$.
Then, we readily obtain
\begin{align}\label{eq:piflat}
\pi^\star = \pi^\star(\theta) = \arg \min_\pi J(\pi) =
\frac{\gamma\theta}{1+\theta(\gamma-1)},
\end{align}
which satisfies $0\leq \pi^\star\leq 1$ if $0\leq \theta\leq 1$.
Furthermore, we obtain
\begin{align}\label{eq:flattening_J}
	J(\pi^\star(\theta)) = \log\gamma - \theta \log \gamma
\end{align}
and
\begin{align*}
	\frac{\partial}{\partial \theta} J(\pi^\star(\theta)) = \frac{\partial}{\partial \theta}\left(\infdiv{q(\xi\mid \theta)}{p(\xi\mid \pi^\star)} - \log p(\pi^\star\mid \gamma)\right) = -\log\gamma, \quad 0\leq \theta \leq 1.
\end{align*}
Thus, the flattening hyper-prior results in a constant term in the gradient of the cost function with respect to the $\Theta$ parameters. This constant term can be easily adjusted by selecting $\gamma$ and used to transparently trade-off the level of network pruning versus loss of prediction accuracy. The reader is referred to \cite{GUENTER2024_robust} for more details in the previous derivations and discussion on the choice of the flattening hyper-prior.

Next, the regularization term on the network's weights in \eqref{eq:Objective_with_KL} comes from the $\log$-probability of the Gaussian prior on $W$ and is expressed as
\begin{align}\label{eq:W-prior}
\log p(W\mid \lambda) = const. -\frac{\lambda}{2}W^\top W,
\end{align}
where $W$ are vectors consisting of the network weights of all nonlinear and linear paths, respectively.

We also express the negative of the integral term in  \eqref{eq:Objective_with_KL} as
\begin{align}\label{eq:Cost}
C(W,\Theta)= \E_{\substack{\Xi\sim  q(\Xi\mid \Theta)}}\big[-\log p(Y\mid X,W, \Xi)\big] =N\cdot \E_{\substack{\Xi\sim  q(\Xi\mid \Theta) \\ (x_i,y_i) \sim p(\mathcal{D})}}\left[-\log p(y_i\mid x_i,W, \Xi)\right],
\end{align}
where $p(\mathcal{D})$ denotes the empirical distribution assigning probability $\frac{1}{N}$ to each sample $(x_i, y_i)$ in the given the dataset $\mathcal{D}=\lbrace (x_i,y_i)\rbrace_{i=1}^N$.
Then, by replacing each $\pi^l$ with the corresponding optimal $\pi^\star(\theta^l)$ from \eqref{eq:piflat} and using \eqref{eq:flattening_J}, \eqref{eq:W-prior} and \eqref{eq:Cost} in \eqref{eq:Objective_with_KL}, dropping constant terms and  switching from maximization to equivalent minimization,  our optimization objective becomes
\begin{align}\label{eq:L_train_obj}
\min_{W,\Theta\in\mathcal{H}} L(W,\Theta) \equalhat C(W,\Theta) + \frac{\lambda}{2}W^\top W-\| \Theta \|_1 \log\gamma.
\end{align}
where $\mathcal{H}$ is the hypercube
\begin{align*}
		\mathcal{H}= \left\{\begin{bmatrix}
		\theta^1,\dots,\theta^L
		\end{bmatrix}\in\mathbb{R}^L\mid0\leq\theta^l\leq 1, \, \forall \,l=1\dots L\right\}.
\end{align*}

\section{Main Results}\label{sec:Main-Results}
\subsection{Characterization of the Solutions of $\min L$}\label{sec:characterization_solutions}
We can express the optimization objective in \eqref{eq:L_train_obj} as:
\begin{align}\label{eq:L_full}
\begin{split}
	&L(W,\Theta) =	\sum_I   g_I(\Theta) \log p(Y\mid X,W,\Xi=I)
	+ \frac{\lambda}{2}W^\top W-\| \Theta \|_1 \log\gamma
\end{split}
\end{align}
where $I=(i_1,\dots,i_L)$ is a multi-index with $i_l=0$ or $1$ defining a particular dropout network and
\begin{align*}
g_I(\Theta) \equalhat 	\prod_{l=1}^L ({\theta^l})^{i_l}(1-\theta^l)^{1-i_l}.
\end{align*}
Then, it follows from \eqref{eq:L_full} that $L(W,\Theta)$ is a multilinear function of $\Theta$, which is box-constrained, i.e. $0\leq\theta^l\leq 1$. This property of \eqref{eq:L_full} is guaranteed by the choice of the flattening hyper-prior that adds only linear terms  with respect to $\Theta$ in the objective and has the profound consequence that the optimal network is deterministic, meaning that the optimal $\theta^l$ and the corresponding RV $\xi^l$ in a layer of the optimal network will be either $0$ or $1$. We formally state this result as the following theorem and provide a simple argument for its validity.

\begin{theorem}
Consider the Neural Network \eqref{eq:struct_res} and the optimization problem \eqref{eq:L_train_obj}.
Then, the minimal values of \eqref{eq:L_train_obj} with respect to $\Theta$ are achieved at the extreme values $\theta^l=0,\ 1$ and the resulting optimal network is deterministic.
\end{theorem}
\begin{proof}
Assume that a (local) minimum value of $L(W,\Theta)$ is achieved for $0<\theta^l<1$. By fixing all $W$, $\theta^k$, $k\neq l$, we obtain a linear function of $\theta_l$. Then, assuming $\frac{\partial L}{\partial\theta^l}\neq 0$, $L$ can be further (continuously) minimized by moving $\theta^l$ towards one of $0$ or $1$, which is a contradiction. If $\frac{\partial L}{\partial\theta^l}=0$, $L$ does not depend on $\theta^l$ and we can replace $\theta^l$ with $0$ or $1$, achieving the same minimal value.
\end{proof}

Dropout \cite{Srivastava_2014_Dropout} has been shown to yield well regularized networks. Initialized with $0<\theta<1$, the networks trained with our method enjoy the same regularizing properties as dropout networks during the early stages of training. As the training continues and the $\theta$ parameter approach $0$ or $1$, a well regularized but smaller deterministic network is found that optimally trades-off size/capacity and performance.
We emphasize again that the usage of the flattening hyper-prior in \eqref{eq:obj_pi} guarantees this result, and therefore, not only it plays a critical role for the robustness properties of the resulting pruning/learning algorithm as elaborated in \cite{GUENTER2024_robust}, but also it assures a deterministic optimal network, which is attractive for the reasons discussed above.

\subsection{Optimality Conditions}\label{sec:KKT}
In the following, we show independently of the previous result that the solutions to \eqref{eq:L_train_obj} satisfy $\theta^\star = 1$ or $0$ and also derive specific conditions that distinguish the two cases.	It is advantageous to consider the reduced optimization problem of minimizing
$L(W,\Theta)$ in \eqref{eq:L_train_obj} only  with respect to the  weights $W_2^l$ and parameters $\theta^l$ associated with layer $l$.

We denote collectively the RVs $\xi^{l'}, l'\neq l$ and parameters $\theta^{l'}, l'\neq l$ as $\bar \Xi^l$ and $\bar \Theta^l$, respectively.
Using the fact that $q(\xi^l\mid \theta^l)$ is a Bernoulli distribution with parameter $\theta^l$, we express $C(W, \theta^l, \bar\Theta^l)$ from \eqref{eq:Cost} in a form that exposes its dependence on $\theta^l$ as follows:
\begin{align*}
&C(W,\theta^l,\bar\Theta^l) = \E_{\Xi\sim  q(\Xi\mid \Theta)}\big[-\log p(Y\mid X,W, \Xi)\big]\\
&\hspace{-0.1cm}= \theta^l \underbrace{\E_{\bar \Xi^l\sim  q(\bar \Xi^l\mid \bar \Theta^l)}\hspace{-0.05cm}\big[\hspace{-0.13cm}-\log p(Y\mid X,W,\xi^l=1,\bar \Xi^l)\big]}_{\equalhat C_1^l(W,\bar\Theta^l)} \hspace{-0.03cm}+(1-\theta^l)\underbrace{\E_{\bar \Xi^l\sim  q(\bar \Xi^l\mid \bar \Theta^l)}\hspace{-0.05cm}\big[\hspace{-0.13cm}-\log p(Y\mid X,W,\xi^l=0, \bar\Xi^l)\big]}_{\equalhat C_0^l(W,\bar\Theta^l)}.\nonumber
\end{align*}

Further, we write the weight matrix $W_2^l$ as a column vector $w^l$:
\begin{align*}
W_2^l = \begin{bmatrix}
\rule[.5ex]{1.5em}{0.1pt}&w_1^{lT}&\rule[.5ex]{1.5em}{0.1pt}\\&\vdots&\\\rule[.5ex]{1.5em}{0.1pt}&w_K^{lT}&\rule[.5ex]{1.5em}{0.1pt}
\end{bmatrix} \quad \quad w^l=[w_1^{lT} \dots w_K^{lT}]^T,
\end{align*}
where the column vectors $[w_k^l]_{1\leq k\leq K}$ are the fan-out weights to the units in the next layer.
Then, we express $L(w^l,\theta^l)$ as a function of $w^l$ and $\theta^l$ only:
\begin{align*}
L(w^l,\theta^l) = \theta^lC_1^l(w^l) + (1-\theta^l)C_0^l - \theta^l\log\gamma + \frac{\lambda}{2}\norm{w^l}^2 + const.,
\end{align*}
and note that, by inspection of eq. \eqref{eq:struct_res}, $C_0^l$ is independent of $w^l$.
For the remainder of this section, we use $w=w^l$, $\theta=\theta^l$ and $C_0^l=C_0$, $C_1^l=C_1$ omitting dependence on layer $l$  for a cleaner presentation.

Therefore, the reduced minimization objective becomes
\begin{align} \label{eq:L_train_obj_k_layer}
\begin{split}
\min_{w, \theta} L(w,\theta) &= \theta\underbrace{\left(C_1(w) -C_0 - \log\gamma \right)}_{\equalhat A(w)} + \underbrace{\frac{\lambda}{2}\norm{w}^2 + const.}_{\equalhat B(w)}\\
&\quad s.t. \quad 0\leq\theta\leq 1.
\end{split}
\end{align}
We discuss later how the implications of solving \eqref{eq:L_train_obj_k_layer} carry over to the full optimization problem \eqref{eq:L_train_obj}.
From \eqref{eq:L_train_obj_k_layer}, we obtain the Lagrangian
\begin{align*}
	\mathcal{L}(w,\theta,\mu_0,\mu_1) = \theta A(w) + B(w) -\mu_0\theta+\mu_1(\theta-1),
\end{align*}
where the Lagrange multipliers $\mu_0$ and $\mu_1$ correspond to the constraints $\theta>0$ and $\theta<1$, respectively.
The necessary first order optimality (Karush-Kuhn-Tucker) conditions are (see, for example \cite[pp.342-345]{Lueneberger_Lin_Nonlin_Prog}:
\begin{align}\label{eq:FONC}
	\begin{split}	
	\frac{\partial \mathcal{L}}{\partial w} = \theta\nabla A(w) + \nabla B(w) = 0,\\
	\frac{\partial \mathcal{L}}{\partial \theta} = A(w) -\mu_0 +\mu_1 = 0,\\
	\mu_0,\mu_1 \geq 0 ,\\
	\mu_0\theta = 0,\, \mu_1(\theta-1)=0.
	\end{split}
\end{align}
Further, we consider the Hessian of the Lagrangian
\begin{align*}
	\nabla^2 \mathcal{L}(w,\theta,\mu_0,\mu_1) = \begin{bmatrix}
	\theta\nabla^2A(w)+\nabla^2 B(w) & \nabla A(w) \\
	\nabla A(w)^T  & 0
	\end{bmatrix}
	=\begin{bmatrix}
	\theta\nabla^2 C_1(w) + \lambda I & \nabla C_1(w)\\
	\nabla C_1(w)^T & 0
	\end{bmatrix}
\end{align*}
and also obtain the second order necessary optimality condition as
\begin{align}\label{eq:SONC}
		s^T\nabla^2 \mathcal{L}(w,\theta,\mu_0,\mu_1)s \geq 0 \quad \forall\, s\neq0, s \in \mathcal{M}
\end{align}
where $\mathcal{M}$ is the tangent space of the active optimization constraints
at a solution  $(w,\theta,\mu_0,\mu_1)$ of \eqref{eq:FONC}.
Switching to strict inequality in \eqref{eq:SONC} (and assuming no degeneracy in the active constraints) provides a sufficient condition for $(w,\theta,\mu_0,\mu_1)$ to be a local minimum point.

Next, let $(w^\star,\theta^\star,\mu_0^\star,\mu_1^\star)$ satisfy \eqref{eq:FONC} and \eqref{eq:SONC} and consider the following three cases:
\paragraph{Case 1, $0<\theta^\star<1$:}
Since $0<\theta^\star<1$, it follows that $\mu_0^\star=\mu_1^\star=0$ and to satisfy the first order optimality conditions \eqref{eq:FONC}, we must have $A(w^\star)=0$ and $\nabla A(w^\star) = - \frac{\nabla B(w^\star)}{\theta^\star}=- \frac{\lambda} {\theta^\star}w^* \neq 0$, because otherwise $w^\star=0$ and hence $C_1=C_0\Rightarrow A(w^\star)=-\log\gamma\neq 0$ for $0<\gamma<1$, which is a contradiction.
From the second order necessary condition, we have (no active constraints)
\begin{align*}
	s^T\nabla^2 \mathcal{L}(w^\star,\theta^\star,\mu_0^\star,\mu_1^\star)s \geq 0 \quad \forall\, s\neq0.
\end{align*}
Next, consider a nonzero entry  $\left[\nabla A(w^\star)\right]_i\neq 0$ in $\nabla A(w^\star)$ and the sub-matrix
\begin{align*}
	\begin{bmatrix}
	m_{ii}^\star & \left[\nabla A(w^\star)\right]_i \\ \left[\nabla A(w^\star)\right]_i & 0
	\end{bmatrix}\quad \text{with} \quad m_{ii}^\star = \left[\theta^\star\nabla^2A(w^\star)+\nabla^2 B(w^\star)\right]_{ii},
\end{align*}
the $i$th diagonal entry of $\theta^\star\nabla^2A(w^\star)+\nabla^2 B(w^\star)$. The determinant of this matrix is a principal minor of the Hessian $\nabla^2\mathcal{L}(w^\star,\theta^\star,\mu_0^\star,\mu_1^\star)$ of the reduced optimization problem \eqref{eq:L_train_obj_k_layer} as well as of the Hessian of the full optimization problem \eqref{eq:L_train_obj}.
It holds
\begin{align*}
	det\left(\begin{bmatrix}
	m_{ii}^\star & \left[\nabla A(w^\star)\right]_i \\ \left[\nabla A(w^\star)\right]_i & 0
	\end{bmatrix}\right) = -\left[\nabla A(w^\star)\right]_i^2 < 0
\end{align*}
and therefore $\nabla^2\mathcal{L}(w^\star,\theta^\star,\mu_0^\star,\mu_1^\star)$, as well as the Hessian of the overall optimization problem \eqref{eq:L_train_obj}, can not be positive semidefinite and the second order necessary optimality condition \eqref{eq:SONC} fails.
As a result, ``Case 1: $0<\theta^\star<1$'' does not produce any local minima for the reduced optimization problem \eqref{eq:L_train_obj_k_layer}, as well as the overall optimization problem \eqref{eq:L_train_obj}.

\paragraph{Case 2, $\theta^\star=0$:}
It follows that $\mu_1^\star=0$ and from the first order optimality conditions we obtain $\nabla B(w^\star) = \lambda w^\star = 0 \Leftrightarrow w^\star=0$ and also $\mu_0=A(w^\star)=C_1(w^\star)-C_0-\log\gamma=-\log\gamma>0$ for any $0<\gamma<1$ since for $w^\star=0$, it is $C_1(w^\star)=C_0$. Furthermore,
the tangent space is $\mathcal{M}=\{s \mid\begin{bmatrix}
0&\dots&0&-1
\end{bmatrix}s=0\}$, i.e., it consists of vectors $s=[s_w^T\ \,\, 0]^T$, where $s_w$ is any vector of the same dimension as $w$. It holds
\begin{align*}
s^T\nabla^2 \mathcal{L}(w^\star,\theta^\star,\mu_0^\star,\mu_1^\star)s = s_w^T \left(\lambda I\right)s_w > 0,
\end{align*}
$\forall s\neq 0\Leftrightarrow \forall s_w\neq 0$ since $\lambda > 0$ and the second order sufficient optimality conditions are satisfied. Therefore, $(\theta^\star=0,\ w^\star=0)$ is a local minimum of the reduced optimization problem \eqref{eq:L_train_obj_k_layer} and indeed of the full optimization problem \eqref{eq:L_train_obj} since the optimality conditions hold independently of the values of the remaining NN weights and $\theta$-parameters.

\paragraph{Case 3, $\theta^\star=1$:}
It follows that $\mu_0^\star=0$ and from the first order optimality conditions, we obtain  $\nabla A(w^\star)+\nabla B(w^\star) = \nabla C_1(w^\star)+ \lambda w^\star = 0$ and also
$\mu_1=-A(w^\star)=-C_1(w^\star)+C_0+\log\gamma\geq 0 \Leftrightarrow
C_1(w^\star)-C_0\leq\log\gamma$.
The tangent space is $\mathcal{M}=\{s\mid\begin{bmatrix}
0&\dots&0&1
\end{bmatrix}s=0\}$ and it consists of vectors $s=[s_w^T\ 0]^T$, where $s_w$ is any vector of the same dimension as $w$. Then, the second order necessary optimality conditions require
\begin{align*}
		s^T\nabla^2 \mathcal{L}(w^\star,\theta^\star,\mu_0^\star,\mu_1^\star)s = s_w^T \left(\nabla^2 C_1(w^\star) + \lambda I\right)s_w\geq0,\ \ \forall s_w.
\end{align*}
In summary, the necessary optimality conditions when $\theta^\star=1$ are:
\begin{itemize}
	\item[O1)] $\nabla^2 C_1(w^\star) + \lambda I \succeq 0$,
	\item[O2)] $A(w^\star) = -\mu_1^\star \leq 0 \Leftrightarrow C_1(w^\star)-C_0 \leq \log\gamma$,
	\item[O3)]	$\nabla C_1(w^\star)+\lambda w^\star = 0$.
\end{itemize}
The above conditions are necessary for $(w^\star, \theta^\star=1)$ to be a local minimum of both the reduced \eqref{eq:L_train_obj_k_layer} and full \eqref{eq:L_train_obj} optimization problems, although they are not expected to characterize the optimal $w^\star$ and $\theta^\star$ for \eqref{eq:L_train_obj_k_layer} independently of remaining NN weights and  $\theta$ parameters.
Assuming that the matrix in O1) is strictly positive definite implies along with O2) and O3) that $(w^\star, \theta^\star=1)$ is a local minimum of  \eqref{eq:L_train_obj_k_layer}.
Also, Condition O2) has the interpretation that surviving layers, i.e. layers with $\theta^\star=1$, should have a cost benefit of at least $-\log\gamma>0$.

\section{Learning Algorithm}\label{sec:algorithm}
We now turn to formulating the learning algorithm for practically find solutions of \eqref{eq:L_train_obj}. We employ projected stochastic gradient descent (SGD) to minimize $L(W,\Theta)$ in \eqref{eq:L_train_obj}
over the NN weights $W$ and variational parameters $\Theta$ as exact solutions are intractable. The projection is  with respect to the Bernoulli parameters $\Theta$ and it ensures that they remain in $[0,1]$.
The required gradients are computed in the following subsections.
\subsection{Learning the Network's Weights $W$}
From \eqref{eq:L_train_obj}, we obtain the gradient of $L(W,\Theta)$ with respect to weights $W^l_1, W^l_2$ and $W_3^l$ of the $l$th layer as
\begin{align}\label{eq:grad_w}
	\frac{\partial L(W,\Theta)}{\partial W^l_j} =\frac{\partial C(W,\Theta)}{\partial W^l_j} + \lambda W^l_j\, , \,\, j=1,2,3
\end{align}
To calculate the first term in \eqref{eq:grad_w}, we first estimate the expectation over the RVs $\Xi$ in \eqref{eq:Cost} with a sample $\hat\Xi\sim Bernoulli(\Theta)$ and since typically the dataset is large, we also approximate the $\log$-likelihood in \eqref{eq:Cost} with a sub-sampled dataset (mini-batch) $\mathcal{S} = \lbrace (x_i,y_i)\rbrace_{i=1}^B$ of size $B$. Thus, we have
\begin{align}\label{eq:MiniBatch_approx}
	C(W,\Theta) \approx \frac{N}{B}\sum_{i=1}^{B}\big[-\log p(y_i\mid x_i,W,\hat\Xi)\big].
\end{align}
We employ a common sample $\hat\Xi$ for all mini-batch samples to benefit from some computational savings \cite{graham2015efficient} and since we noticed no perceptible difference in our simulations when using different $\hat\Xi_i$ realizations for each mini-batch sample.
We define the gradient of the partial cost $-\log p(y_i\mid x_i,W,\hat\Xi)$  with respect to signals $z^l$ and $\bar z^l$ as
\begin{equation*}
	\delta^l_i\equiv\frac{\partial \left(-\log p(y_i\mid x_i,W,\hat\Xi)\right)}{\partial z^l}
	\quad\text{and}\quad \bar\delta^l_i\equiv\frac{\partial \left(-\log p(y_i\mid x_i,W,\hat\Xi)\right)}{\partial \bar z^l},\quad l=1,\ldots,L,
	\label{delta_l}
\end{equation*}
and obtain from a straightforward adaptation of the standard backpropagation algorithm (see for example \cite[Chapter 5]{bishop_pattern_2006}) for the residual-type network structure \eqref{eq:struct_res} the following recursive relations:
\begin{align*}
	\delta^L_i &=W^{L\top}\bar\delta^L_i;\nonumber\\
	\bar\delta^l_i &=\delta^{l+1}_i;\quad
	\delta^l_i =\hat\xi^l\left(W_1^{l\top} \text{diag}\left\{{a^l}^{\prime}\left(W_1^l z^l\right)\right\}\right)W_2^{l\top}\bar\delta^l_i + V^{l\top}\delta^{l+1}_i \quad \quad\text{for}\quad l=L-1, \dots, 1.
\end{align*}
Here, $\delta_i^l$, $\bar\delta_i^l$ are column vectors and computation of  $\bar\delta_i^{L}$  involves the output activation and loss function considered,  that is, the underlying statistical model.
More specifically, for both the regression and classification tasks with linear output activation/normal output distribution and softmax output activation/categorical output distribution, we have from \eqref{eq:model_regression} and \eqref{eq:model_classification}, respectively:
\begin{align*}
	\bar\delta^L_i\equiv\frac{\partial \left(-\log p(y_i\mid x_i,W,\hat\Xi)\right)}{\partial \bar z^L} & =y_i-\hat{y}_i.
\end{align*}
Then, the gradients of the objective $L(W,\Theta)$ with respect to the network weights $W^l$ can be approximated by
\begin{align}\label{eq:grad_mL}
	\frac{\partial L(W,\Theta)}{\partial W^L} &\approx \widehat{\frac{\partial L}{\partial W^L}}= \frac{N}{B}\sum_{i=1}^{B} \bar\delta^{L}_i \cdot{z^L_i}^\top + \lambda W^L,\quad  \\
	\frac{\partial L(W,\Theta)}{\partial W^l_1} &\approx \widehat{\frac{\partial L}{\partial W_1^l}}=\frac{N}{B}\sum_{i=1}^{B} \hat \xi^l \cdot{\left( \text{diag}\left\{{a^l}^{\prime}\left(W_1^l z^l\right)\right\} W_2^{l\top} \bar \delta^l_i {z}^{l\top}\right)} + \lambda W^l_1,\quad l=1,\ldots,L-1\\
	\frac{\partial L(W,\Theta)}{\partial W^l_2} &\approx \widehat{\frac{\partial L}{\partial W_2^l}}=\frac{N}{B}\sum_{i=1}^{B} \hat \xi^l\bar\delta^{l}_i \cdot{a^l\left(W_1^lz^l\right)}^\top + \lambda W^l_2,\quad l=1,\ldots,L-1\label{eq:grad_mL2}
\end{align}
and
\begin{align}\label{eq:grad_m}
	\frac{\partial L(W,\Theta)}{\partial W_3^l} \approx \widehat{\frac{\partial L}{\partial W_3^l}}=\frac{N}{B}\sum_{i=1}^{B}\delta^{l+1}_i \cdot {z^l_i}^\top + \lambda W_3^l,\quad l=1,\ldots,L-1.
\end{align}

\subsection{Learning the Pruning Parameters $\theta$}\label{sec:learning_theta}\label{sec:deltaC_taylor}
From \eqref{eq:L_train_obj_k_layer}, we readily compute the derivative with respect to $\theta^l$ as
\begin{align}\label{eq:grad_theta}
\frac{\partial L(W,\theta^l,\bar\Theta^l)}{\partial \theta^l} =  C_1^l-C_0^l -\log\gamma.
\end{align}
We remark that the gradient in  \eqref{eq:grad_theta} depends on $\log\gamma$ rather than explicitly on $\gamma$, which makes possible to numerically tolerate small values of $\gamma$ for achieving appropriate regularization.

In \eqref{eq:grad_theta}, $C_1^l-C_0^l$ is the difference in the total cost with the particular nonlinear layer path switched on and off.
A large negative value for the difference $C_1^l-C_0^l$ indicates high importance of this path for the performance of the network.
Then in minimizing $L(W,\theta^l,\bar\Theta^l)$ via gradient descent, a negative value for $C_1^l-C_0^l-\log\gamma$ will make the corresponding $\theta^l$ grow, leading to the path being switched on and its weights being optimized more frequently. On the other hand, a positive $C_1^l-C_0^l-\log\gamma$ leads to smaller $\theta^l$ and less frequent optimization of the weights of the corresponding path and as a consequence the weight decay term drives these weights to zero.

Calculating the expected values $C_1^l,\,C_0^l$ requires evaluating the network with all combinations of paths switched on and off. Depending on the total number of layers $L$, this may be computationally expensive, making it necessary to approximate them instead.
In the following, we approximate the gradient of \eqref{eq:L_train_obj}  with respect to $\theta^l$ as
\begin{align}\label{eq:rho_grads}
	\frac{\partial L} {\partial \theta^l}\approx \widehat{\frac{\partial L} {\partial \theta^l}} = \widehat{\Delta C^l} - \log\gamma, \quad \text{with} \quad \widehat{\Delta C^l} \approx C_1^l-C_0^l,
\end{align}
where we obtain the estimate $\widehat{\Delta C^l}$ from Monte Carlo estimates over data $\mathcal{D}$ and $\Xi$ of the cost in \eqref{eq:MiniBatch_approx}. Specifically, each mini-batch computation gives an unbiased estimate of $C_1^l$ or $C_0^l$ depending on whether $\hat \xi^l=1$ or $0$:
\begin{equation}\label{eq:Chats_sampling}
	C_{j}^l \approx \hat{C}_{j}^l =\frac{N}{B} \sum_{i=1}^B -\log p(y_i \mid x_i,W,\hat{\xi}^l=j, \hat{\bar{\Xi}}^l), \quad j=0,1.
\end{equation}
Subsequently, the network is evaluated a second time switching the value of $\hat \xi^l$ while the sample $\hat{\bar\Xi}^l$ is kept the same. This approach requires  $L+1$ evaluations of the network to obtain unbiased estimates of $C_1^l$ and $C_0^l$ for all layers, where $L$ is the number of hidden layers in the network.
Then, the estimator for the difference is:
\begin{align}\label{eq:Sampling_estimator}
	\begin{split}
		C_1^l-C_0^l \approx& \widehat{\Delta C^l}= \frac{N}{B} \sum_{i=1}^B \hspace{-0.1cm}-\log p(y_i \mid x_i,W,\hat \xi^l=1, \hat{\bar\Xi}^l) \hspace{-0.05cm}- \frac{N}{B} \sum_{i=1}^B \hspace{-0.1cm}- \log p(y_i \mid x_i,W,\hat \xi^l=0, \hat{\bar\Xi}^l)\\
		&=\frac{N}{B} \sum_{i=1}^B \log \left(\frac{p(y_i \mid x_i,W,\hat \xi^l=0, \hat{\bar\Xi}^l)}{p(y_i \mid x_i,W,\hat\xi^l=1, \hat{\bar\Xi}^l)}\right).
	\end{split}
\end{align}
While this unbiased estimator was found to have relatively low variance in practice and the $L+1$ evaluations of forward passes of the network can be done in parallel, it may still be computationally unattractive for very deep networks even though the required number of forward network evaluations is linear in the number of layers of the network. Then, the difference $C_1^l-C_0^l$ can be approximated, for example, by a first order Taylor approximation
as shown before in \cite{GUENTER2024_robust} for the case of unit pruning.
We extend the Taylor approximation of $\widehat{\Delta C^l}$ for the case of layer pruning as follows:
\begin{align*}
	\widehat{\Delta C^l} \approx \Delta\xi^l \frac{\partial C(W,\theta^l,\bar\Theta^l)}{\partial \xi^l}
	=\Delta\xi^l\frac{\partial\E_{\bar\Xi^l}\big[-\log p(Y\mid X,W, \bar\Xi^l,\xi^l)\big]}{\partial\xi^l},
\end{align*}
where for $\xi^l=0$, $\Delta\xi^l=1$ and for $\xi^l=1$, $\Delta\xi^l=-1$. In both cases, estimating the above expectation via sample means over the mini-batch data, yields
\begin{align}\label{eq:Taylor_estimator}
	C_1^l-C_0^l \approx \widehat{\Delta C^l}=\frac{N}{B}\sum_{i=1}^B \bar\delta_i^lW_2^lz_i^l,
\end{align}
with $\bar\delta_i^l$ and $z_i^l$ depending on $\hat{\Xi}$. The estimator in \eqref{eq:Taylor_estimator} is equivalent to the Straight-Through estimator proposed in \cite{bengio_estimating_2013}. It is biased after dropping the higher order terms in the Taylor series expansion, however, a more detailed analysis in \cite{GUENTER2024_robust} shows that for $W_2^l\rightarrow 0$, this estimator is asymptotically unbiased.

\subsection{Formulation of the Stochastic Gradient Descent Algorithm}
We now present the proposed concurrent learning and pruning algorithm based on minimizing the objective in  \eqref{eq:L_train_obj}.
This algorithm generates discrete-time sequences $\{W^l_1(n), W^l_2(n), W_3^l(n), \theta^l(n)\}_{n\geq0}$ for each nonlinear layer path of the network to minimize the objective \eqref{eq:L_train_obj} based on projected stochastic gradient descent.
More specifically, we consider the sequences
\begin{align}\label{eq:DE_W}
W^L(n+1) = W^L(n)-a(n)\widehat{\frac{\partial L}{\partial W^L}},
\end{align}
and for all $l=1,\dots, L-1$, $j=1,2,3$
\begin{align}\label{eq:DE}
\begin{split}
W_j^l(n+1) &= W_j^l(n)-a(n)\widehat{\frac{\partial L}{\partial W_j^l}},
\quad\text{and}\quad
\theta^l(n+1) =\Pi_\mathcal{H}\left[\theta^l(n)-a(n)\widehat{\frac{\partial L} {\partial \theta^l}}\right]
\end{split}
\end{align}
where the estimates of the gradient of the objective function $L(W(n),\Theta(n))$ are obtained from \eqref{eq:grad_mL}-\eqref{eq:grad_m}  and \eqref{eq:rho_grads} together with \eqref{eq:Sampling_estimator} or \eqref{eq:Taylor_estimator} in practice.
The projection $\Pi_\mathcal{H}$ can be equivalently written as
\begin{align}\label{eq:DE_th}
\begin{split}
	\theta^l(n+1) =\Pi_\mathcal{H}\left[\theta^l(n)-a(n)\widehat{\frac{\partial L} {\partial \theta^l}}\right] =  &\max\left\{\min\left\{\theta^l(n)-a(n)\widehat{\frac{\partial L} {\partial \theta^l}},\,1\right\}, \,0\right\}.
\end{split}
\end{align}
Finally, we assume that the stepsize $a(n)$ satisfies the Robbins-Monro conditions:
\begin{equation}\label{eq:RM}
\sum_{n=0}^\infty a(n) = \infty  \quad \text{and}\quad \sum_{n=0}^\infty a(n)^2 < \infty.
\end{equation}
Pseudo-code of the algorithm is given in Algorithm \ref{alg:learning_algo}.
The algorithm assumes a given dataset: $\mathcal{D}=\lbrace (x_i,y_i)\rbrace_{i=1}^N$, mini-batch size $B$ and an initial network with the structure of \eqref{eq:struct_res} characterized by the number of layers and the number of units in each layer. We select hyper-parameters $0<\gamma<1$ for the flattening hyper-priors and $\lambda>0$ for weight regularization.
We remark that as the dataset size $N$ increases and because $C_0^l$ and $C_1^l$ are total expected errors over the dataset, $C_0^l-C_1^l$ roughly increases linearly with $N$. This behavior is consistent with the Bayesian approach whereby using more data samples to estimate the posterior distribution leads to a diminishing influence of prior information. Therefore, to induce sufficient pruning and for the previous pruning conditions to be meaningful, the hyper-parameter $\log\gamma$ for the flattening hyper-prior  needs to be matched appropriately to the size of the dataset.

Each iteration of the algorithm consists of five main steps. In Step~1, a mini-batch $\mathcal{S}=\lbrace ( x_i, y_i)\rbrace_{i=1}^B$ is sampled from the data $\mathcal{D}$ with replacement and a sample $\hat \Xi \sim Bernoulli(\Theta)$ is obtained and used to predict the network output and to approximate needed expectations. In Step~2, the gradients of the objective \eqref{eq:L_train_obj} with respect to the network weights $W$ are computed via backpropagation and $\frac{\partial C}{\partial\theta}=C_1^l-C_0^l$ is approximated for each layer using \eqref{eq:Sampling_estimator} or the Taylor approximation \eqref{eq:Taylor_estimator} from Section~\ref{sec:deltaC_taylor}. Then, the gradients of the objective with respect to $\theta$ are obtained from \eqref{eq:rho_grads}. Step~3 constitutes the learning phase in which the network weights $W$ and variational parameters $\Theta$ are updated via gradient descent utilizing the previously computed gradients. The convergence results of Section~\ref{sec:convergence_results} require that the stepsize $a(n)$ for the gradient descent update satisfies the Robbins-Monro conditions \eqref{eq:RM}
and that the $\theta$-parameters are projected onto the interval $[0,1]$ after the update step.

In Step~4, we identify the layers that can be pruned away and remove them from the network to reduce the computational cost in further training iterations.
Based on Theorem~\ref{thm:Alg_conv} in Section \ref{sec:convergence_results}, layers can be safely removed from the network if the weights $w$ and update rate $\theta$ of a nonlinear layer path enter and remain in the region $\mathcal{A}_0$ defined in \eqref{eq:RoA}, since then convergence $\{w\rightarrow 0,\, \theta \rightarrow 0\}$ is guaranteed for this layer.
This result is of clear theoretical value but difficult to utilize in practice. However, it points to more practical conditions for layer removal that we have found to work well. Specifically, it is shown in Section~\ref{sec:convergence_results} (see \eqref{eq:Sys_GrDes}), that $\theta\rightarrow 0$ implies $w\rightarrow 0$. Then, we assume that convergence for a layer has been achieved once its $\theta\leq\theta_{tol}$, a user-defined parameter, and absorb the layer for the remaining iterations.

Finally in Step~5, if the algorithm has converged i.e., the estimated gradients of the objective \eqref{eq:L_train_obj} with respect to $W$ and $\Theta$ are small, or a maximum number of iterations has been reached, we exit the training process. Otherwise, we set $n=n+1$ and repeat the previous five steps.

We remark that our algorithm is applicable to both, fully connected and convolutional networks. Although, we present details for the fully connected case for reasons of brevity, extension of our algorithm to the case of a convolutional layer simply entails the introduction of a Bernoulli random variable $\xi$ for the collection of filter matrices in a layer. Then, $\xi$ multiplies again the nonlinear layer path output and if $\xi$ is zero, this path is inactive. Only minor modifications are required for the gradient computation of the performance objective with respect to the variational parameters; these gradients are used for learning the posterior distributions over the random variables $\xi$ and in turn for selecting which of the convolutional layers to prune and which to keep. In Section~\ref{sec:simulations},  we apply the algorithm also to convolutional neural networks and compare its performance with competing methods.

\begin{algorithm}
	{
		\caption{Learning/Pruning Algorithm}
		\label{alg:learning_algo}
		\begin{algorithmic}[1]	
			\REQUIRE  Dataset $\mathcal{D}=\lbrace (x_i,y_i)\rbrace_{i=1}^N$ and mini-batch size $B$, Initial Network Structure with Initial Weights $W(0), V(0)$ and parameters $\Theta(0)$;
			Hyper-parameters:  $\log\gamma<0$, $\lambda>0$, $\theta_{tol}>0$. \vspace{-0.3cm}
			\\ \hrulefill \\
			\STATE initialize n=0
			\WHILE{Training has not converged or exceeded the maximum number of iterations}
			\vspace{-0.3cm}
			\STATE 
			\hrulefill \\
			\textbf{STEP 1: Forward Pass}\\
			\STATE Sample $B$ times with replacement from the dataset $\mathcal{D}$ to obtain samples $\mathcal{S}=\lbrace ( x_i, y_i)\rbrace_{i=1}^B$.
			\STATE Sample a common realization of the network by sampling $\hat \Xi \sim Bernoulli(\Theta)$.
			\STATE Using current weights $W(n)$, predict the network's output $\hat y_i = NN(x_i;W,\hat \Xi) \ \forall i=1\dots B$ as in \eqref{eq:struct_res}.
			\vspace{-0.3cm}
			\\ \hrulefill \\
			\textbf{STEP 2: Backpropagation Phase}\\
			\STATE Approximate the gradient  with respect to weights $g^l_{W_j} = \widehat{\frac{\partial L(W,\Theta)}{\partial W^l_j}}\, , j=1,2,3, \ \forall l=1\dots L$  using  \eqref{eq:grad_mL}-\eqref{eq:grad_m} with the sample $\hat \Xi$ and $\mathcal{S}$.
			\STATE Approximate $\frac{\partial C}{\partial\theta}=C_1^l-C_0^l$ for each layer in the network as described in Section \ref{sec:deltaC_taylor}.
			\STATE Approximate the elements of the gradient  with respect to $\theta^l$, $g_\theta^l = \widehat{ \frac{\partial L(W,\Theta}{\partial \theta^l}} \ \forall l=1\dots L$ with  \eqref{eq:grad_theta} and using the previous approximations of $C_1^l-C_0^l$.
			\vspace{-0.3cm}
			\\ \hrulefill \\
			\textbf{STEP 3: Learning Phase}\\
			\STATE  Take gradient steps $W^l_j(n+1) = W^l_j(n) - a(n) g^l_{W_j}, \, j=1,2,3$  and\\ $\theta^l(n+1) = \theta^l(n) - a(n) g_\theta^l$ for appropriate step size $a(n)$ satisfying: $\sum_{n=0}^\infty a(n) = \infty$ and $\sum_{n=0}^\infty a(n)^2 < \infty.$\\
			\STATE Project $\theta^l(n+1) = \max\left\{\min\left\{\theta^l(n+1),\,1\right\}, \,0\right\}$
			\vspace{-0.3cm}
			\\ \hrulefill \\
			\textbf{STEP 4: Network Pruning Phase}\\
			\FOR{each hidden layer $l$}
			\IF { $\theta(n+1)<\theta_{tol}$}
			\STATE Prune the  layer by setting its weights to zero and removing it from the network (see eq. \eqref{eq:lin_layer_reduction}).
\STATE If possible, absorb the linear path of layer $l$ to the preceding or following layer.
			\ENDIF
			\ENDFOR
			\vspace{-0.3cm}
			\\ \hrulefill \\
			\textbf{STEP 5: Check for Convergence}
			\STATE If the magnitude of the gradients $g^l_W$ and $g_\theta^l$ is less than a specified tolerance, the algorithm has converged.
			\\ \hrulefill \\
			\STATE set $n=n+1$
			\ENDWHILE
		\end{algorithmic}
	}
\end{algorithm}

\section{Convergence Results}\label{sec:convergence_results}
In this section, we derive convergence results for the algorithm proposed in Section  \ref{sec:algorithm} by employing the continuous-time ordinary differential equations (ODE) defined by the expectations of the stochastic gradients.
Following \cite[pp.106-107]{kushner2003stochastic}, we construct and consider the ODE:
\begin{align*}
\left\{\dot W_j^l =-\E_{\Xi,\mathcal{D}} \left[ \widehat{\frac{\partial L}{\partial W_j^l}}\right],
 \quad \dot \theta^l =-\E_{\Xi,\mathcal{D}} \left[ \widehat{\frac{\partial L}{\partial\theta^l}}  \right]-h^l,\ \ l=1,\dots, L-1,\ \ j=1,2,3 \right\},
\end{align*}
with $\widehat{\frac{\partial L}{\partial W_j^l}}$, $\widehat{\frac{\partial L}{\partial\theta^l}}$ as defined in Section~\ref{sec:algorithm} and
\begin{align}\label{eq:hl}
	h^l \equiv \begin{cases}
	-\left[C_1^l-C_0^l -\log\gamma\right]_{+} \quad &\theta=0\\
	0 &0<\theta<1\\
	\left[C_1^l-C_0^l -\log\gamma\right]_{-} \quad &\theta=1
	\end{cases},
\end{align}
with $\left[\cdot\right]_+$, $\left[\cdot\right]_-$ denoting the positive and negative part of their argument, respectively.

Since the estimated gradients are unbiased (see discussion in Section~\ref{sec:learning_theta} and \ref{app:assumption_verification}), we have:
\begin{align}\label{eq:ODE_all_weights_unb}
\left\{\dot W_j^l =-\frac{\partial L}{\partial W_j^l},
 \quad \dot \theta^l =-\frac{\partial L}{\partial \theta^l}-h^l,\ \ l=1,\dots, L-1,\ \ j=1,2,3 \right\}.
\end{align}

Next, we establish that the stationary points of the ODE \eqref{eq:ODE_all_weights_unb} are the same with the stationary points of the optimization problem \eqref{eq:L_train_obj}. The stationary points of the ODE satisfy:
\begin{align}\label{eq:ODE_stationary}
\begin{split}
\dot W_j^l =-\frac{\partial L}{\partial W_j^l}=0,\ \ j=1,2,3\\
\dot \theta^l =-\frac{\partial L}{\partial \theta^l}-h^l=0,
\end{split}
\end{align}
for $l=1,\dots, L-1$.
On the other hand, expressing the Lagrangian of \eqref{eq:L_train_obj} in a similar fashion with the analysis of the reduced optimization problem \eqref{eq:L_train_obj_k_layer}, we obtain the first order optimality conditions defining the stationary points of \eqref{eq:L_train_obj} as:
\begin{align}\label{eq:FONC_full}
	\begin{split}	
	\frac{\partial L}{\partial W^l_j} = 0,\ j=1,2,3\\
	\frac{\partial L}{\partial \theta^l} -\mu_0^l +\mu_1^l = 0,\\
	\mu_0^l,\mu_1^l \geq 0 ,\\
	\mu_0^l\theta^l = 0,\, \mu_1^l(\theta^l-1)=0,
	\end{split}
\end{align}
for $l=1,\dots, L-1$. Further, proceeding as in Section~\ref{sec:KKT} (see Cases~1 to~3), we can establish that
\begin{align}\label{eq:mus}
\mu_0^l=\left\{\begin{array}{lc} [C_1^l-C_0^l-\log\gamma]_{+}&\ \theta^l=0\\
0&\  0<\theta^l\leq 1\end{array}\right. &\ \rm{and}\ &
\mu_1^l=\left\{\begin{array}{ll}  [C_1^l-C_0^l-\log\gamma]_{-}&\ \theta^l=1\\ 0 &\  0\leq\theta^l< 1\end{array}\right.
\end{align}
Then, comparing \eqref{eq:hl} and \eqref{eq:mus}, it is clear that the solutions of \eqref{eq:ODE_stationary} and \eqref{eq:FONC_full} expressing the stationary points of the ODE \eqref{eq:ODE_all_weights_unb} and the solutions of the optimization problem \eqref{eq:L_train_obj} are identical.

\subsection{Convergence to Solutions of \eqref{eq:L_train_obj} and Deterministic Networks}
The next theorems present the main convergence results.
\begin{theorem}\label{thm:Alg_conv}
	Consider the sequence $x(n)\equiv\{W_1^l(n),W_2^l(n),W_3^l(n),\theta^l(n), \, \forall l=1,\dots, L-1\}_{n\geq0}$ of all weights and their update rates of the Neural Network \eqref{eq:struct_res} as generated by \eqref{eq:DE_W}-\eqref{eq:DE_th} with parameter $\lambda>0$ and $\log\gamma<0$ and a step size satisfying \eqref{eq:RM}.
	Assume that all weights of the network $W$ remain bounded, i.e., assume for all $n$
\begin{align}\label{eq:weight_bound}
	\norm{W^l_j(n)} \leq \phi_{max}<\infty, \, j=1,2,3,\ \ l=1,\ldots,L-1.
\end{align}
 Further, assume that all activation functions used in the Neural Network \eqref{eq:struct_res} are continuously differentiable with bounded derivative $a^{l'}(\cdot)$ and satisfy $a^l(0)=0$. Lastly, assume that the data satisfies
	\begin{align}\label{eq:data_moments}
	\begin{split}
		\E_{x\sim p(x)}\norm{x}^k \leq S_{xk} < \infty,\,\, k=1,2,3,4 \quad
		\text{and} \quad \E_{y\sim p(y|x)}\norm{y}^k\leq S_{yk} < \infty, \,\, k=1,2.
	\end{split}
	\end{align}
	Then, the limit of a convergent subsequence of $x(n)$ satisfies the projected ODE \eqref{eq:ODE_all_weights_unb} and $x(n)$ converges almost surely to a stationary set of the ODE in $\mathcal{H}$.	
\end{theorem}
\begin{proof}
	The proof of Theorem~\ref{thm:Alg_conv} is based on well-established stochastic approximation results in \cite[Theorem 2.1, p.127]{kushner2003stochastic}.  We verify the required assumptions of Theorem 2.1 in \cite{kushner2003stochastic}  for the learning algorithm \eqref{eq:DE_W}-\eqref{eq:DE_th} and \eqref{eq:RM} in \ref{app:assumption_verification}.
\end{proof}
We remark that the assumption on the boundedness of the weights is reasonable since the $\mathcal{L}_2$ weight regularization terms in objective \eqref{eq:L_train_obj} should keep the norm of the weights from becoming excessively large. However, this assumption can be removed by using projection also on the weights at the expense of a more complicated argument.

\begin{theorem}\label{thm:conv_min}
	Consider the same setup and assumptions of Theorem~\eqref{thm:Alg_conv}. In addition, assume that the noise introduced to the gradient estimates in \eqref{eq:DE_W}-\eqref{eq:DE_th} by the stochastic gradient descent process  is isotropic with nonzero variance for all $n$. Then, $x(n)$ converges almost surely  to a solution (local minimum) of the minimization problem \eqref{eq:L_train_obj}, representing a deterministic Neural Network.
\end{theorem}
\begin{proof}
	Theorem \ref{thm:Alg_conv} establishes convergence of the algorithm to a stationary set of the ODE in $\mathcal{H}$. We have already shown that the stationary points of the ODE match those of the optimization problem.
	As a result, if the ODE converges to a stationary set, it converges to a stationary set of the objective \eqref{eq:L_train_obj}.
	For any stationary set with a $0<\theta^{l,\star}<1$, we have shown in Section~\ref{sec:characterization_solutions}, Case~1 that the Hessian of the optimization problem \eqref{eq:L_train_obj} has at least one negative eigenvalue, i.e., it is a strict saddle point or a local maximum.
	Then, application of Theorem 8.1 in \cite[p.159]{kushner2003stochastic} shows that under the stated noise assumption in Theorem~\ref{thm:conv_min}, the algorithm will escape any such stationary point and, therefore, cannot converge to any point with $0<\theta^{l,\star}<1$. The same reasoning also shows that the descent process will escape local maxima with $\theta^{l,\star}=0,\rm{\ or\ }1$.
	As a result, the algorithm \eqref{eq:DE_W}-\eqref{eq:DE_th} will converge to a (local) minimum of \eqref{eq:L_train_obj} at which for all layers $l$ it is either $\theta^{l,\star}=0$ or $1$, meaning that all Bernoulli parameters of the Neural Network \eqref{eq:struct_res} are either $0$ or $1$, leading to a deterministic network.	
\end{proof}

It has been established that perturbed gradient descent escapes saddle points efficiently \cite{Jin_escape2017}. Similarly, it has been shown that in machine learning tasks, such as training a DNN, Stochastic Gradient Descent is capable of escaping saddle points and maxima without requiring any additional perturbation \cite{Daneshmand1_escape2018}. This implies that sufficiently exciting noise is introduced to the gradient estimates automatically and the noise assumption in Theorem~\ref{thm:conv_min} should be typically satisfied.

\subsection{Convergence of Individual Layers Leads to Concurrent Training and Pruning.}
In this section, we analyze individual layers of the network in the ODE system \eqref{eq:ODE_all_weights_unb}, and for clarity dependence on layer $l$ in our notation is omitted. Following the notation of Section~\ref{sec:KKT}, we denote the fan-out weights of the nonlinear path in layer $l$ by $w$ and the parameter of its Bernoulli RV by $\theta$. Then, the ODE system \eqref{eq:ODE_all_weights_unb} reduces for $w$ and $\theta$ as follows:
\begin{align}\label{eq:Sys_GrDes}
\begin{split}
\dot w &= -\frac{\partial L}{\partial w}=-\theta \nabla A(w)-\lambda w\\   	
\dot \theta  &= -\frac{\partial L}{\partial \theta}-h=-\underbrace{[C_1(w)-C_0-\log\gamma]}_{A(w)}-h.
\end{split}
\end{align}
with $h$ from \eqref{eq:hl}.

The following result establishes that $\{w=0,\, \theta=0\}$ is an asymptotically stable equilibrium point of the dynamical system \eqref{eq:Sys_GrDes} (see also Section~\ref{sec:KKT}, Case~2) and provides a subset of its region of attraction (RoA). Most importantly, the RoA is independent of the remaining layers of the network and invoking stochastic approximation theorems, the result can be used to justify the pruning of layers from the NN  during the early stages of training, even when other layers may have not converged yet.

\begin{theorem}\label{thm:stability}
	Consider the dynamical system \eqref{eq:Sys_GrDes} with $\lambda>0$ and $\log\gamma<0$.	
	Assume that all weights of the network remain bounded, i.e. \eqref{eq:weight_bound} holds.
	Then, $\{w=0,\theta=0\}$ is a locally asymptotically stable equilibrium of the system \eqref{eq:Sys_GrDes}.
	Moreover,
	\begin{align}\label{eq:RoA}
	\mathcal{A}_0= \left\{w \in \mathbb{R}^q,\, 0\leq \theta \leq 1 \quad \middle|\quad \frac{1}{2}\norm{w}^2 + \frac{1}{2}\theta^2  < \frac{1}{2}\min\left\{\left(\frac{-\log\gamma}{\kappa+\eta}\right)^2, 1\right\} \right\}
	\end{align}
	belongs to its region of attraction.
\end{theorem}
\begin{proof}
	$\{w=0,\theta=0\}$ is an equilibrium point of the dynamics of $w$ and $\theta$ in \eqref{eq:Sys_GrDes} since for $w=0$ it is $C_1-C_0=0$ and then $\log\gamma<0$ would lead to $\theta$ leaving the feasible set $[0,1]$. Therefore, in such a case it is $h=\log\gamma$ and $\dot \theta=0$.
	To establish its local asymptotic stability, we consider the Lyapunov candidate function
	\begin{align*}
	\Lambda(w,\theta) =\frac{1}{2}\norm{w}^2 + \frac{1}{2}\theta^2,
	\end{align*}
	which satisfies $\Lambda>0$ for $\{w,\theta\}\not=\{0,0\}$ and show that $\dot \Lambda<0$ in the region $\mathcal{A}_0$, which is a level set of $\Lambda$.
	We calculate the Lie-Derivative:
	\begin{align*}
	\begin{split}
	\dot \Lambda &=w^\top \dot w  + \theta\dot \theta\\
	&= -\lambda w^\top w -\theta w^\top \nabla A(w) - \theta\left(C_1-C_0\right) + \theta\log\gamma -\theta h.
	\end{split}
	\end{align*}
	Given assumption \eqref{eq:weight_bound}, i.e., all network weights remain bounded, it is shown in \cite{GUENTER2024_robust}, Appendix A, that for standard feed-forward NNs the following bounds hold:
	\begin{align}\label{eq:eta_kappa_bnd}
	|w^T\nabla A(w)| \leq \eta\norm{w} \quad \text{and} \quad |C_1-C_0| \leq \kappa \norm{w}
	\end{align}
	with constants $\eta,\kappa <\infty$ independent of the remaining weights of the network. These bounds are trivially extended to the residual NNs considered in this work, since the skip connection can be thought of a part of the nonlinear layer path but with an identity activation function.
	Then, using the bounds in \eqref{eq:eta_kappa_bnd} the Lie-Derivative is bounded as follows:
	\begin{align*}
	\begin{split}
	\dot \Lambda &\leq -\lambda \norm{w}^2 + \eta\theta\norm{w} + \theta\kappa\norm{w} + \theta\log\gamma -\theta h\\
	&=-\lambda \norm{w}^2  + \theta\left((\eta+\kappa)\norm{w} + \log\gamma - h\right).
	\end{split}
	\end{align*}
	Further if $0 <\theta < 1$, it is $h=0$ and from
	\begin{align}\label{eq:Lneg}
	´\left(\kappa+\eta\right) \norm{w} + \log\gamma < 0 \quad \Leftrightarrow \quad \norm{w} < \frac{-\log\gamma}{\kappa+\eta}
	\end{align}
	it follows that $\dot \Lambda<0$.
	From \eqref{eq:hl} and for $\theta=0$,  if $C_1-C_0-\log\gamma\leq 0$ then  $h=0$ and the case above applies or otherwise $h=-(C_1-C_0-\log\gamma)<0$ and the Lie-Derivative becomes.
	\begin{align*}
	\dot \Lambda  = -\lambda \norm{w}^2 - \theta w^T \nabla A(w) =  -\lambda \norm{w}^2 < 0.
	\end{align*}
	Condition (\ref{eq:Lneg}) is satisfied in the region $\mathcal{A}_0$, which is a level set of $\Lambda$, and thus $\Lambda$ is a local Lyapunov function for the system \eqref{eq:Sys_GrDes}. In addition, $\mathcal{A}_0$ is such that $\theta < 1$. In summary, the analyzed equilibrium point is locally asymptotically stable with region of attraction encompassing $\mathcal{A}_0$.
\end{proof}

We remark that under the assumptions of Theorem~\ref{thm:Alg_conv},  if $(w(n),\ \theta(n))$ for layer $l$ generated by the stochastic gradient descent algorithm \eqref{eq:DE}  visits the RoA established in Theorem~\ref{thm:stability} infinitely often with probability $p$, then $(w(n),\ \theta(n))$ converges to the equilibrium $(w=0,\ \theta=0)$ of the ODE \eqref{eq:Sys_GrDes} with probability at least $p$.  This result follows from  Theorem~2.1 in \cite{kushner2003stochastic} and although it cannot be easily verified, it does provide justification for early pruning of layers. More specifically, our empirical evidence with Algorithm \ref{alg:learning_algo} and the simulations presented in Section~\ref{sec:simulations} suggest that the actual RoA of the equilibrium point $(w=0,\theta=0)$ is much larger than the one given by \eqref{eq:RoA}; furthermore, with typical choices of the step size $a(n)$, we have not observed any $\theta$ parameter recovering from a value of less than $\theta_{tol}=0.01$ and not converging eventually to $0$. Therefore, a decreasing $\theta$ below a value of $0.01$ reliably signals the convergence of the nonlinear path of a layer to its dead equilibrium and its immediate removal.

\section{Simulation Experiments}\label{sec:simulations}
In this Section, we evaluate the proposed simultaneous learning and layer pruning algorithm on several datasets such as the MNIST dataset \cite{lecun-mnisthandwrittendigit-2010}, the CIFAR-10 and CIFAR-100 datasets \cite{Krizhevsky_09} and the ImageNet dataset \cite{imagenet_deng2009}. Starting from different deep residual network architectures or by adding residual paths into existing architectures like, for example, VGG16 \cite{simonyan_deep_2015} to fit our framework, the goal is to learn the appropriate number of layers together with the network's weights and obtain significantly smaller networks having performance on par with baseline performance of the trained larger networks.
Each experiment consists of training on the full training dataset and evaluating the found models on the test dataset. We do not cross-validate during training to select the final network. If possible, we evaluate the robustness of our method, i.e., its sensitivity with respect to weight initialization, by reporting the mean and standard deviation of the results over 10 runs.
The training parameters vary for the different datasets/architectures and are given in the following subsections. The networks resulting after the last epoch of training are saved and evaluated in terms of their structure, accuracy and parameter pruning ratio as well as floating point operations (FLOPS) pruning ratio (fPR). The parameter pruning ratio (pPR) is defined as the ratio of pruned weights over the number of total weights in the starting architecture. Similarly, the FLOPS pruning ratio is the percentage of FLOPS saved by the network found by the algorithm.
All simulations were run in MATLAB\textsuperscript{\tiny\textregistered} or Python\textsuperscript{\tiny\textregistered} using several TensorFlow\texttrademark/Keras libraries.

By neglecting the computations involving activation functions as a small fraction of the total computational load, we can estimate the computational load of the baseline network per iteration by {$\nu\cdot\sum_{l=1}^{L} \nu^l n^{l}n^{l+1}$} where $n^l$ is the number of features/units in the $l$th layer and $\nu$ is proportionality constant; also for a fully connected layer $\nu^l=1$  and for a convolutional layer $\nu^l = d_F^2\cdot d_W \cdot d_H$, where $d_F$ is the size of a square filter (kernel) matrix and $d_W \times d_H$ is the dimension of the resulting 2-D feature map in the $l$th layer.
For the dropout networks used in our algorithm, the expected computational load per iteration is upper-bounded by
\begin{align}\label{eq:comp_load}
\nu\cdot\sum_{l \in I} \nu^l n^{l}n^{l+1},
\end{align}
where $I$ is the set of indices of surviving layers. This  set will decrease as we train and prune the network and hence the computational load per iteration decreases. We use \eqref{eq:comp_load} to calculate how much computational load is saved in various experiments with our method.

\subsection{MNIST Experiments}
The MNIST dataset was designed for character recognition of handwritten digits (0-9) and consists of $60,000$ 28x28 grayscale images for training and an additional $10,000$ alike images for evaluation. The following hyper- and training parameters are used: Dataset size $N=60,000$, Mini-batch size $B=64$, weight $\mathcal{L}_2$-Regularization parameter $\lambda=1$.
We use a $\log\gamma=-200$ for experiments with fully connected NNs and $\log\gamma=-50$ for experiments with convolutional NNs.
All NNs are trained using the Adam optimizer \cite{kingma_adam:_2014} with a learning rate of $\expnumber{1}{-3}$ for 50 epochs and follow-up with a fine-tuning phase by training the pruned network for an additional 10 epochs with a learning rate of $\expnumber{1}{-4}$. Prior to the fine-tuning phase, all $\theta$ values less than $\expnumber{1}{-3}$ are set to $0$ and all other to $1$, thus specifying the final deterministic network architecture.

First, we consider  fully-connected NN architectures of $10$, $20$ and $50$ hidden layers, all with the same width of $100$ units each and therefore allowing for identity skip connections. The number of parameters in these networks are $170410$, $271410$ and $574410$, respectively. We note that the input layer holds $79,285$ weights and we do not prune this layer as it is needed to reduce the dimensionality of the input to a size of $100$.
Table~\ref{tab:archs_FCNN} summarizes the results.  For the $10-$ and $20$-layer networks, our algorithm prunes the network to only a two-layer network in almost all of the $10$ different initializations. This means, that besides the first layer, always only one other layer is kept. We note that this layer was 9 out of 10 times the second layer, independent of the starting number of layers. In the 10th run, the layer kept in the network was the third layer. The test accuracy of all resulting networks is consistently around $98.27\%$ and compares to the accuracy obtained with LeNet300-100, a 2 hidden-layer network with $300$ and $100$ units, of $98.46 \pm 0.09\%$. This network consists of over $260,000$ parameters compared to about $90,000$ parameters left in the architectures found by our algorithm.

\begin{table}[h!]
	\centering
	\begin{tabular}{lclclc}
		Number of Layers &  Layers after Training & Test Accuracy [\%]& pPR [\%] \\
		\hline
		\textbf{Fully-Connected NNs}\\
		\hline
		$10$ & $2\pm0$& $98.27\pm 0.13$& $47.42\pm 0$\\
		$20$ & $2\pm0$ & $98.22 \pm 0.07$& $66.98\pm 0$\\
		$50$ & $2.17\pm0.40$& $98.25 \pm 0.13$& $84.11\pm 0.72 $\\
		\hline
		\textbf{Convolutional NNs}\\
		\hline
		$10$ & $2.20\pm 0.42$& $99.23\pm 0.06$& $32.62\pm 3.68$\\
		$20$ & $2.70\pm 0.82$ & $99.23 \pm 0.07$& $52.59\pm 4.13$\\
		$40$ & $3.00\pm 0.67$& $99.32 \pm 0.07$& $71.34\pm 2.32 $\\
		\hline
	\end{tabular}
	\centering
	\caption{Resulting architectures, test accuracy and pruning ratio for the networks obtained by the learning/pruning method starting from different sized initial fully-connected and convolutional deep networks on MNIST.\label{tab:archs_FCNN}}
\end{table}

Next, NNs with $10$, $20$ and $40$ convolutional layers with 5x5 filter matrices subject to pruning are considered. The first half of these layers are with 6 features while the second half are with 16 features. In every network, the convolutional layers are followed by two fully-connected layers of width $120$ and $84$. Again, identity skip connections are used and no residual connection is used when the number of features increases from 6 to 16 and also in the fully connected layers. In principle, these architectures are deeper versions of the LeNet5 architecture which consists only of two convolutional layers followed by the same fully connected layers as in our networks.
We do not prune the input layer as well as the fully connected layers. Table \ref{tab:archs_FCNN} shows the results. Our algorithm is able to prune the convolutional layers of the networks to only $2.20$ to $3$ layers remaining on average. Test accuracy is consistently at over $99.2\%$. LeNet5 achieves a baseline test accuracy of $99.20\pm 0.08 \%$. In case of the $40$-layer network the number of parameters is reduced from around $183,500$ to around $52,500$, of which around $40,000$ are in the fully-connected layers at the end of the network. Considering only the parameters of the convolutional layers, our algorithm is able to prune over $92\%$ of them.

In Figure~\ref{fig:CNN_inds}, we plot histograms of the layer indices corresponding to the layers still present in the network after applying our learning and pruning algorithm to the three starting architectures. Since we do not prune the first layer, it is always present in all of the 10 runs and its bar has a height of 10. The two fully connected layers at the end of each network are not shown in these graphs. Our algorithm favors keeping layers in the second half of the network, where each layer holds 16 features. As the initial network depth increases, the histograms become less sharp indicating that for the survival of a layer, its particular position within the second half of the network is not the most important factor, but rather a favorable initialization and how a layer evolves during training dictate whether it is pruned or not.

\begin{figure}[!htb]
	\centering
	\includegraphics[width=0.32\textwidth]{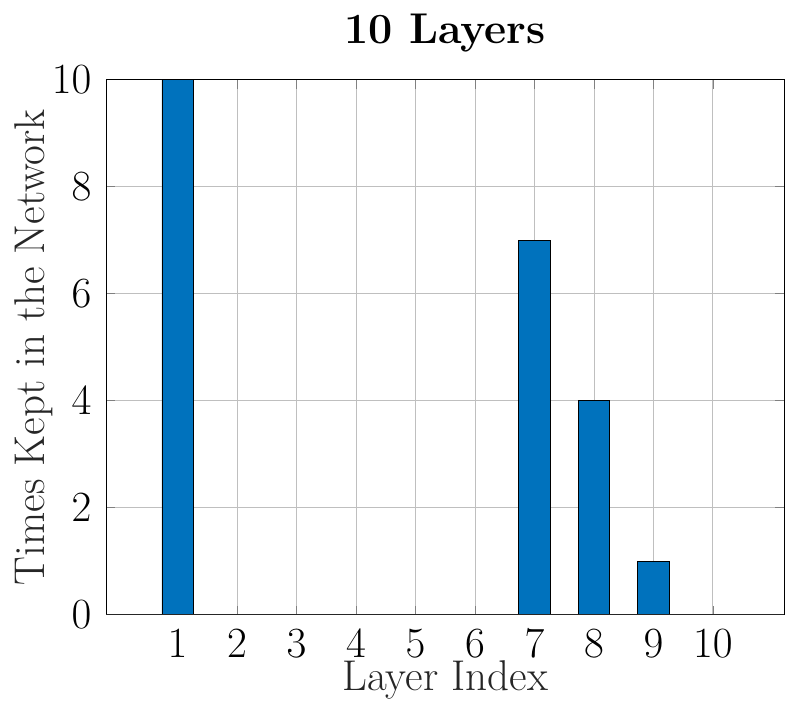}	
	\includegraphics[width=0.32\textwidth]{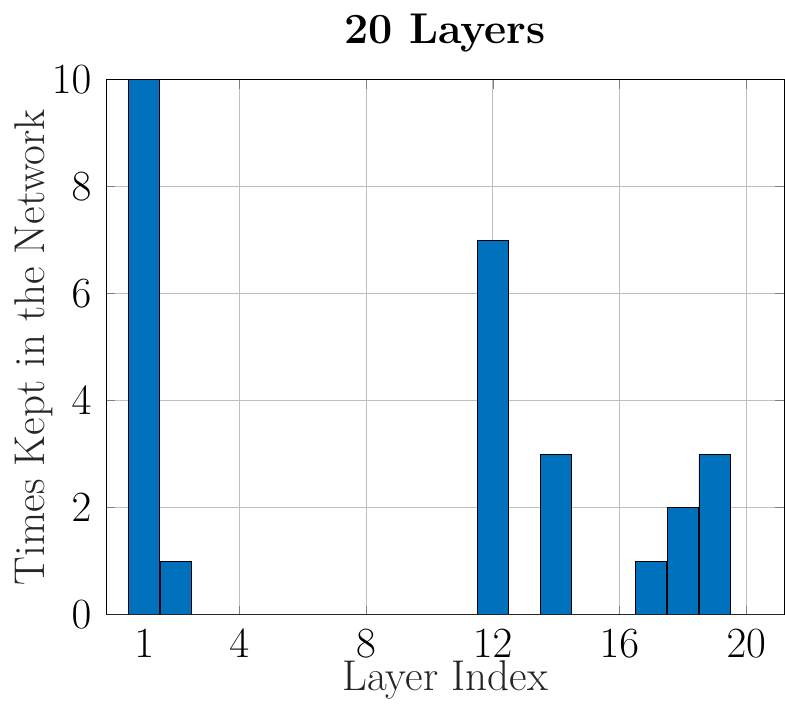}	
	\includegraphics[width=0.32\textwidth]{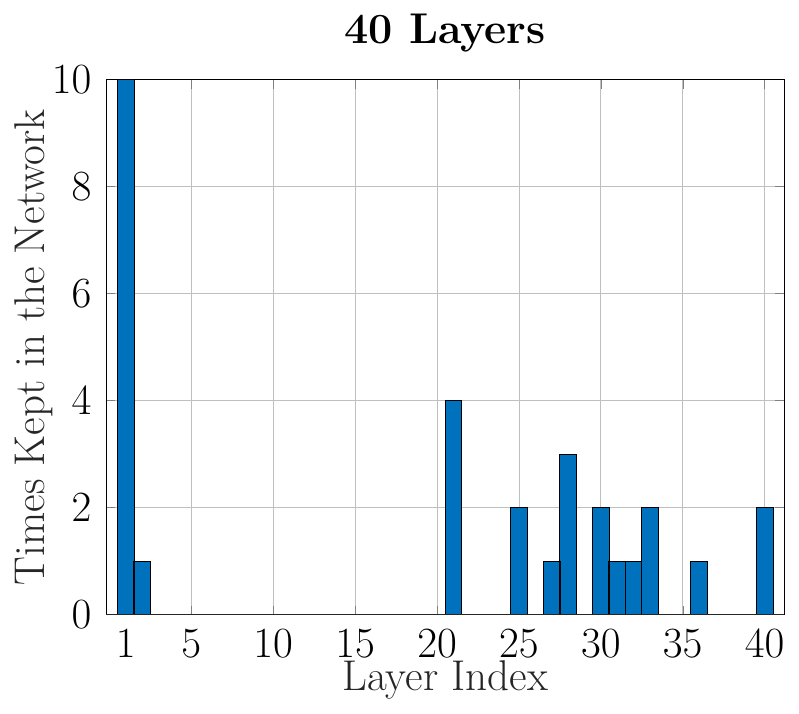}	
	\caption{Histograms of the layer indices that survived the pruning/training process when using our algorithm with the convolutional NN architectures on the MNIST dataset. }\label{fig:CNN_inds}
\end{figure}

\subsection{CIFAR-10/100 Experiments}
We evaluate our algorithm on the CIFAR-10 and CIFAR-100 dataset \cite{cifar10} on common architectures such as VGG16 \cite{simonyan_deep_2015} and ResNet \cite{ResNet_he_2016} as well as our network structure defined in eq. \eqref{eq:struct_res} and Table~\ref{tab:architectures}.
All $50,000$ training images of the CIFAR-10/100 datasets are used during training and the network is evaluated on the $10,000$ test images of the 10/100-class dataset.

\subsubsection{VGG16}
The VGG16 network structure \cite{simonyan_deep_2015}, consists of 13 convolutional and 2 or 3 dense layers, summing up to around $15$ million parameters.
To fit the VGG16 network \cite{li_pruning_2017} in our framework (see eq. \eqref{eq:struct_res}), we introduce identity skip-connections between each two convolutional layers according to Table \ref{tab:architectures}. If the number of filters increases by a factor of two, we simply stack a copy of the output of the previous layer onto itself to match the dimensions of the next layer. No skip connections are introduced for the two dense layers in the network, i.e., we do not aim to prune these layers. Experiments with this modified VGG16 network will be labeled ``Ours'' in the following.

The following training parameters and hyper-parameters are used: Dataset size $N=50,000$, Mini-batch size $B=128$, $\mathcal{L}_2$-Regularization parameter $\lambda= 150$. All $\theta$ parameters are initialized at $0.75$ and we choose $\log\gamma=-1200, -1800$. The networks are trained for 300 epochs in total.
The Adam optimizer \cite{kingma_adam:_2014} with constant learning rate of $\expnumber{1}{-4}$ is used to learn the $\theta$ parameters. To learn the networks weights, we use SGD  with momentum of $0.9$ and initial learning rate of $0.01$ and a cosine decay schedule. At epoch 210 all $\theta$ parameters less than 0.1 are set to 0 and all other to 1.

Table \ref{tab:archs_VGG16} shows our pruning results for the two choices of $\log\gamma$ and lists also the results of \cite{chen19_shallowing} and \cite{xu2020layer} using their (modified) VGG16 architecture for pruning on the same dataset.
Our method is able to prune around 73\% of the networks parameters with only a 0.2\% drop in test accuracy for the choice $\log\gamma=-1200$. However the FLOPS reduction is only around 33\%. For $\log\gamma=-1800$ the parameters are reduced by over 91\% and the FLOPS by over 63\%, yielding a computationally cheaper network than, for example, ResConv-Prune-B \cite{xu2020layer}  at the expense of a somewhat reduced test accuracy. In \cite{sparse_huang2018}, it is reported that only a reduction of 30\% of parameters can be achieved with minor loss of test accuracy using their method and a pruning ratio of about 66\% leads to decrease in test accuracy of over 2.5\%.

\begin{table}[h!]
	\centering
	\begin{tabular}{lclclclc}
		\hspace{-0.05cm}Method &  Test Acc. [\%]&Baseline [\%] &pPR[\%] &fPR[\%] \\
		\hline
		Ours, $\log\gamma=-1200$ & $93.86$& $94.05$ & $72.85$& $33.12$\\
		Ours, $\log\gamma=-1800$ & $92.74$& $94.05$ & $91.55$& $63.22$\\

		\hline
		\cite{chen19_shallowing}  & $93.40$&$93.50$ & $87.90$& $38.9$\\	

		\hline
		ResConv-Prune\cite{xu2020layer}-A   & $94.65$&$94.15$ & $63.0$& $21.2$\\
		ResConv-Prune\cite{xu2020layer}-B   & $93.45$&$94.15$ & $75.9$& $55.9$\\
		ResConv-Prune\cite{xu2020layer}-C  & $92.71$&$94.15$ & $92.1$& $73.3$\\
	\end{tabular}
	\centering
	\caption{Resulting architecture and test accuracy and pruning ratio of the resulting network of the learning/pruning method on the modified VGG16 architecture on the CIFAR-10 dataset.\label{tab:archs_VGG16}}
\end{table}

Figure \ref{fig:VGG16_param_load} shows on the left the number of parameters remaining in the network during training. Our method reduces the number of parameters within the first 60 epochs to around 4 and 1.3 million for $\log \gamma = -1200$ and $-1800$, respectively.
The plot on the right hand side in Figure \ref{fig:VGG16_param_load} shows the computational load over 300 training epochs and based on \eqref{eq:comp_load}.
Training the baseline VGG16 network \cite{simonyan_deep_2015}  for 300 epochs on the CIFAR-10 dataset is around $1.43$ and  $2.30$ times computationally more expensive than our algorithm for $\log\gamma=-1200$ and $\log\gamma=-1800$, respectively.

\begin{figure}[!htb]
	\centering
	\includegraphics[width=0.4\textwidth]{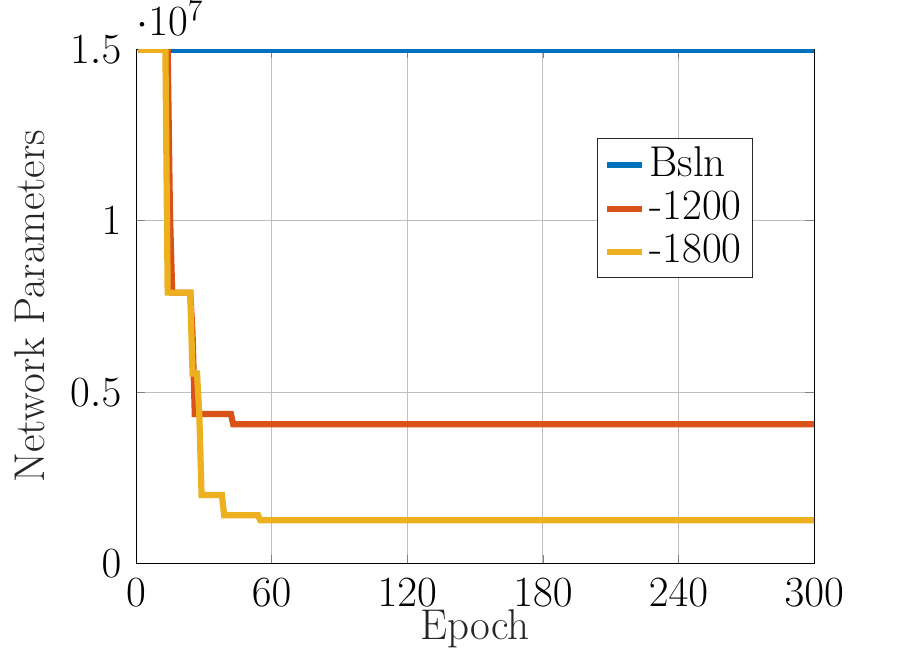}	
	\includegraphics[width=0.4\textwidth]{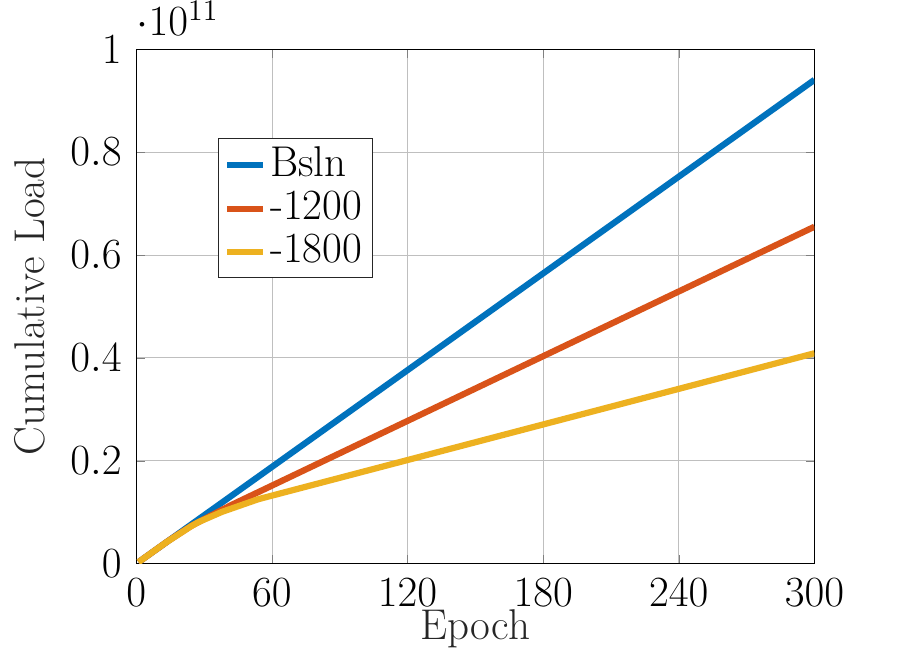}		
	\caption{Results for our version of VGG16 on CIFAR-10. Left: The total number of parameters in the network during training/pruning.
		Right: Cumulative training load of the network during training/pruning. }\label{fig:VGG16_param_load}
\end{figure}

\subsubsection{ResNet56 and ResNet110}
We evaluate our method on the ResNet56 and ResNet110 NNs, consisting of 56 and 110 layers, respectively, as well as on our similar $W$-$B$-$a$-$W$ architecture proposed in Table~\ref{tab:architectures} and  specified in \eqref{eq:struct_res}, on the CIFAR-10 and CIFAR-100 datasets.
All experiments use the following hyperparameters: Dataset size $N=50,000$, Mini-Batch size $B=128$, $\mathcal{L}_2$-Regularization parameter $\lambda=1.5$. All $\theta$ parameters are initialized at $0.75$.
The networks are trained for $182$ epochs.
During the training process $\log\gamma$ is scheduled to increase linearly from $0$ to its final value at epoch $105$ and constant afterwards. We test different final $\log\gamma$ values ranging from $-100$ to $-1200$.
The Adam optimizer \cite{kingma_adam:_2014} with constant learning rate of $\expnumber{1}{-4}$ is used to learn the $\theta$ parameters. To learn the networks weights, we use SGD with momentum of $0.9$ and initial learning rate of $0.1$ and a cosine decay schedule. At epoch 150, all $\theta$ parameters less than $0.1$ are set to 0 and all other to 1; the NNs are fine-tuned for an additional 32 epochs.
We report mean and standard deviation of our results as well as the best individual run directly below for comparison of our method to existing methods for layer pruning.

Table~\ref{tab:CIFAR10_ResNet56} shows the results of our method applied to our  $W$-$B$-$a$-$W$ architecture of equal size as ResNet56. Using $\log \gamma =-500$, our method is able to prune the  $W$-$B$-$a$-$W$ network to only 26 layers while achieving a test accuracy of $93.34\%$. This leads to a FLOPS reduction of $55.46\%$ and to an on-par result with LPSR-50\% \cite{LPSR_zhang} and DBP-A \cite{DBPwang2019}.
For the choice $\log \gamma= -250$, our method is able to reduce the FLOPS by $39.48\%$ while achieving an accuracy of $93.81\%$ and improves upon the result of LPSR-30\% \cite{LPSR_zhang} and ResConv-Prune-A \cite{xu2020layer} in both metrics. It is clear how the choice of $\log\gamma$ affects prediction performance and network size.

\begin{table}[h!]
	\centering
	\begin{tabular}{ccccccccc}
		Method & Test Acc. [\%] & Baseline [\%] &Layer Left & fPR[\%] & pPR[\%]\\
		
		\hline	
		 Ours, $\log\gamma=-250$& $93.66\pm0.23$  & $94.26\pm0.15$ & $34.4\pm2.1$ &  $38.73\pm 3.88$ & $28.24\pm 4.15$ \\
		Best Run 	& $93.81$  & $94.26\pm0.15$ & $34$ &  $39.48$ & $34.20$ \\
		
		 Ours, $\log\gamma=-500$& $93.10\pm0.17$  & $94.26\pm0.15$ & $26.2\pm1.5$ &  $54.52\pm2.73$ & $38.83\pm6.66$ \\
		Best Run & $93.34$  & $94.26\pm0.15$ & $26$ &  $55.46$ & $36.93$ \\
		
		 Ours, $\log\gamma=-800$& $92.64\pm0.45$  & $94.26\pm0.15$ & $21.6\pm2.1$ &  $62.89\pm3.87$ & $46.10\pm 4.94$ \\
		Best Run	& $92.94$  & $94.26\pm0.15$ & $22$ &  $62.04$ & $48.87$ \\
		
		 Ours, $\log\gamma=-1200$& $92.45\pm 0.43$  & $94.26\pm0.15$ & $19.4\pm1.9$ &  $66.93\pm3.57$ & $51.20\pm 6.14$ \\
		Best Run		& $92.53$  & $94.26\pm0.15$ & $18$ &  $69.56$ & $51.58$ \\			
			
		\hline
		\cite{xu2020layer}ResConv-Prune-A & 93.75  & 94.12 &- &  36.0 & 28.2\\
		\cite{xu2020layer}ResConv-Prune-B & 92.72  & 94.12 &- &  63.4 & 58.8\\
		\cite{xu2020layer}ResConv-Prune-C & 91.59  & 94.12 &- &  73.5 & 65.9\\
		\hline
		\cite{chen19_shallowing}  & 93.09  & 93.03 &- &  34.8 & 42.3\\

		\hline

		\cite{LPSR_zhang}LPSR-30\%& 93.70  & 93.21 &40 & 30.1 & 23.5\\
		\cite{LPSR_zhang}LPSR-50\%& 93.40  & 93.21 &28 &  52.68 & 44.89\\
		\cite{LPSR_zhang}LPSR-70\%& 92.34  & 93.21 &18 &  71.5 & 61.2\\
	
		\hline
		\cite{DBPwang2019}DBP-A& 93.39  & 93.72 &- &  53.41 & -\\
		\cite{DBPwang2019}DBP-B& 92.32  & 93.72 &- &  68.69 & -\\

	\end{tabular}
	\centering
	\caption{Resulting test accuracy, the number of layers in the network, FLOPS pruning ratio and parameter pruning ratio of the pruned networks found with our method applied to the  $W$-$B$-$a$-$W$ version of ResNet56 on the CIFAR-10 dataset in comparison to competing methods for layer pruning.} \label{tab:CIFAR10_ResNet56}
\end{table}

Table \ref{tab:CIFAR10_ResNet110} shows the results of training our version of ResNet110 on the CIFAR-10 dataset for $\log\gamma=-150, -250$ and $-800$.
In the first case, the resulting network achieves a test accuracy of $94.25\%$, which virtually equals the test accuracy of the baseline ResNet56 network on the same task. However, the pruned network resulting from ResNet110 consists of about 51 layers and a computational load of $116.4$ million FLOPS, which is significant lower when compared to about $125$ million FLOPS of the baseline ResNet56. Thus, our method has been able to identify a smaller structure that performs equally well as ResNet56 when it was given the flexibility of a larger initial network.

For $\log\gamma=-250$, our method achieves a test accuracy of $93.91\%$
while pruning the network's computations by $64.37\%$, which results in a considerably smaller network than ResConv-Prune-A \cite{xu2020layer} and DBP-A \cite{DBPwang2019}.
For $\log \gamma=-800$, $79.30\%$ of the network's computations and $70.47\%$ of the network's parameters were pruned, while maintaining a test accuracy of around $93\%$.

\begin{table}[h!]
	\centering
	\begin{tabular}{ccccccccc}
		Method & Test Acc. [\%] & Baseline [\%] &Layer Left & fPR[\%] & pPR[\%]\\
		
		\hline
		Ours, $\log\gamma=-150$& $94.25\pm0.19$  & $94.70\pm0.20$ & $51.2\pm3.2$ &  $53.97\pm 2.90$ & $44.35\pm 2.94$ \\
		Best Run 	& $94.34$  & $94.70\pm0.20$ & $46$ &  $58.78$ & $47.43$ \\

		Ours, $\log\gamma=-250$& $93.86\pm0.17$  & $94.70\pm0.20$ & $41.4\pm3.1$ &  $63.11\pm 3.02$ & $52.66\pm 2.81$ \\
		Best Run 	& $93.91$  & $94.70\pm0.20$ & $40$ &  $64.37$ & $53.86$ \\
		
		Ours, $\log\gamma=-800$& $92.83\pm0.14$  & $94.70\pm0.20$ & $24.0\pm1.3$ &  $79.44\pm1.26$ & $68.63\pm2.86$ \\
		Best Run & $93.02$  & $94.70\pm0.20$ & $24$ &  $79.30$ & $70.47$ \\
		
		\hline
		\cite{xu2020layer}ResConv-Prune-A &$94.01$ & $94.15$& - &  $23.7$ & $8.7$\\
		\cite{xu2020layer}ResConv-Prune-B & $93.79$& $94.15$& - &  $65.6$& $55.5$\\
		\cite{xu2020layer}ResConv-Prune-C & $93.67$& $94.15$& - &  $67.8$& $69.4$\\
		
		\hline
		\cite{DBPwang2019}DBP-A & $93.61$ & $93.97$ & - & $43.92$ & - \\
		\cite{DBPwang2019}DBP-B & $93.25$& $93.97$ & -  & $56.98$& -

	\end{tabular}
	\centering
	\caption{Resulting test accuracy, the number of layers in the network, FLOPS pruning ratio and parameter pruning ratio of the pruned networks found with our method applied to the  $W$-$B$-$a$-$W$ version of ResNet110 on the CIFAR-10 dataset in comparison to competing methods for layer pruning.} \label{tab:CIFAR10_ResNet110}
\end{table}

Figure~\ref{fig:R110_C10_param_load} shows on the left the total number of parameters remaining in the network and on the right the cumulative computational load to train the networks calculated using \eqref{eq:comp_load} over the 182 epochs for our ResNet110 experiments. Both plots show the mean over the 10 runs of each experiment.
From the left plot, it can be seen that most pruning is done within the first 60 epochs of training. Afterwards, the network is small and the cumulative computational load increases almost linearly at a lower rate when compared to the baseline network. Training the baseline ResNet110 network is 3.53, 2.07 and 1.66 times more expensive than simultaneously training and pruning the networks with our method  using $\log \gamma= -150, -200$ and $-800$, respectively.

\begin{figure}[!htb]
	\centering
	\includegraphics[width=0.4\textwidth]{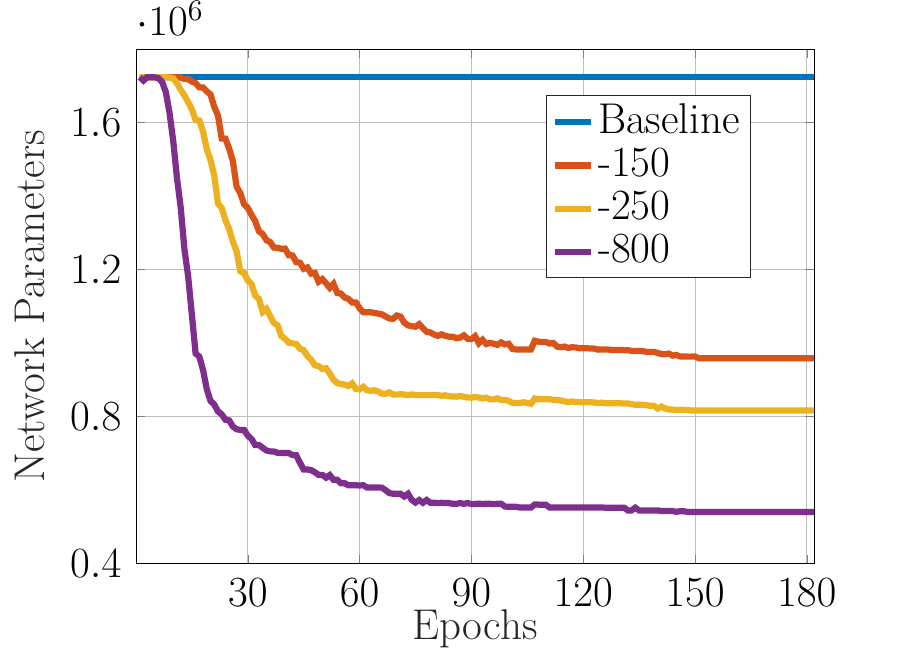}	
	\includegraphics[width=0.4\textwidth]{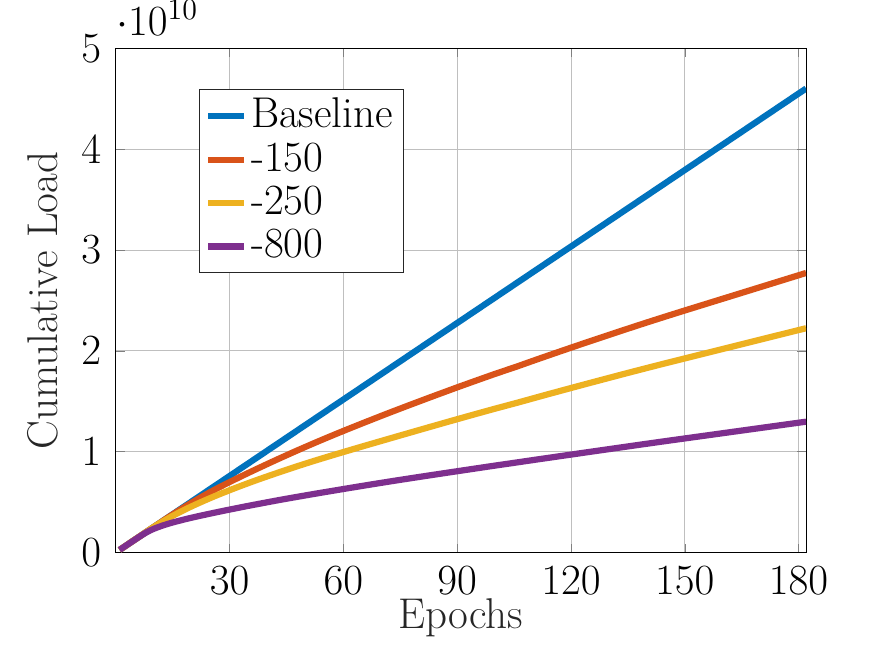}
	\caption{Results for our  $W$-$B$-$a$-$W$ version of ResNet110 on CIFAR-10. Left: The total number of parameters in the network during training/pruning.
		Right: Cumulative training load of the network during training/pruning. }\label{fig:R110_C10_param_load}
\end{figure}

Table~\ref{tab:CIFAR100_ResNet} shows the results on the CIFAR-100 dataset for ResNet56 and ResNet110. We noticed that for CIFAR-100 using a common value of $\log\gamma$ for all layers leads to the deeper and wider layers of the network always surviving. These layers, however, contain the most parameters and are attractive to prune. Therefore, we multiply $\log\gamma$ times $5$ for all layers with 64 filters and indicate such experiments with the additional label ``x5''. Our method is able to prune 12 layers from ResNet56 with only a 0.14\% loss in accuracy for $\log\gamma=-250$. For a more aggressive $\log\gamma=-1400$, our method yields a 24-layer network reducing the computational load by around 59\% and the parameters by 52\%, while still achieving a test accuracy of $70.45\%$.
ResNet110 is pruned by around 28\% in computational load with only a 0.16\% loss in accuracy when using $\log\gamma=-100$.
A more heavily pruned network is obtained by using $\log\gamma=-500$. This network is reduced to 5, 8 and 7 residual blocks of width 16, 32 and 64, respectively and is considerably smaller than the baseline ResNet56 network while performing over 1\% point better in test accuracy, at around 74\%.
Again, when given initially more flexibility by starting with a larger initial network architecture, our algorithm not only finds an architecture of suitable size, but also finds superior weights. Training the baseline ResNet110 network is 2.23 times as expensive as concurrently training and pruning our network. Training the baseline ResNet56 network is only about 10\% cheaper than obtaining the better network with our method starting from ResNet110.

\begin{table}[h!]
	\centering
	\begin{tabular}{ccccccccc}
		Method & Test Acc. [\%] & Baseline [\%] &Layer Left &MFLOPS& fPR[\%] & pPR[\%]\\
		
		\hline
		\textbf{ResNet56}\\
		\hline	
		
		 Ours, $\log\gamma=-250$	& $72.57$  & $72.71$ & $44$ & $97.18$  & $22.56$ & $6.53$ \\	
		 Ours, $\log\gamma=-1000$(x5)	& $72.04$  & $72.71$ & $30$ & $65.32$  & $47.94$ & $39.09$\\		
		 Ours, $\log\gamma=-1400$(x5)	& $70.45$  & $72.71$ & $24$ & $51.17$  & $59.22$ & $52.11$ \\				
		\hline
		\cite{chen19_shallowing}& $69.77$ & $70.01$ & - & - & $38.3$ & $36.1$\\
		\cite{LPSR_zhang}LPSR-50\%& $70.17$& $71.39$ & - & $60.2$ &$52.60$  & $37.21$ 	\vspace{0.5cm}\\	
		
		\textbf{ResNet110}\\
		\hline
		 Ours, $\log\gamma=-100$	& $75.34$  & $75.50$ & $80$ & $182.11$  & $27.99$ & $8.06$ \\
		 Ours, $\log\gamma=-100$(x5)	& $75.05$  & $75.50$ & $88$ & $202.17$  & $20.06$ & $19.56$ \\		
		 Ours, $\log\gamma=-500$(x5)	& $73.98$  & $75.50$ & $42$ & $93.64$  & $62.97$ & $60.29$ \\

	\end{tabular}
	\centering
	\caption{Resulting test accuracy, the number of layers in the network, FLOPS, FLOPS pruning ratio and parameter pruning ratio of the pruned networks found with our method applied to the  $W$-$B$-$a$-$W$ versions of ResNet56 and ResNet110 on the CIFAR-100 dataset.} \label{tab:CIFAR100_ResNet}
\end{table}

\subsection{ImageNet Experiments}
In this section, we evaluate the proposed algorithm on the ImageNet \cite{imagenet_deng2009} dataset consisting of a training set of about 1.28 million 224x224 colored images in 1000 categories and a validation set of 50,000 alike images. We employ our $W$-$B$-$a$-$W$ modification of the common ResNet50 residual deep NN architecture with around 25.6 million parameters and a computational load of around 4 billion FLOPS. All training images are used during training and the network is evaluated on all 50000 validation images using the single view classification accuracy.

The following training parameters are used: dataset size N=1,281,167, mini-batch size B=256, $\mathcal{L}_2$-Regularization parameter $\lambda=25$. All $\theta$ parameters are initialized at $0.75$. During the training process, $\log\gamma$ is scheduled to increase linearly from $0$ to its final value of $\log\gamma=-17,500$ at epoch $35$ and remain constant afterward.
The network is trained for $90$ epochs using the AdamW \cite{Loshchilov2017DecoupledWD} optimizer with a cosine decay learning rate schedule. The results from our approach are compared to existing methods in Table~\ref{tab:Imagenet_ResNet50}. Our method is able to prune the network to 17.94 million parameters and a computational load of only 2.8 billion FLOPS while achieving a test accuracy of $74.82\%$. This results falls in between ResConv-Prune-A and B from \cite{xu2020layer} in terms of both, computational load and test accuracy of the pruned network.
Our method achieves a slightly higher test accuracy with slightly higher computational load when compared to  LPSR-40\% \cite{LPSR_zhang} and DBP-A \cite{DBPwang2019}. Compared to \cite{sparse_huang2018}-B, our method finds a similar sized network with considerably higher test accuracy.

\begin{table}[h!]
	\centering
	\begin{tabular}{ccccccccc}
		Method & Test Acc. [\%] & Baseline [\%] & FLOPS (billion)& Parameter (million) \\

		\hline
		 Ours , $\log\gamma=-17,500$ & 74.82  & 76.04  & 2.80  & 17.94  \\
		
		\hline
		\cite{sparse_huang2018}-A & 75.44 & 76.12 & 3.47 & 25.30\\
		\cite{sparse_huang2018}-B & 74.18 & 76.12 & 2.82 & 18.60\\
		\hline
		\cite{xu2020layer}ResConv-Prune-A & 75.44  & 76.15  & 3.65 & 16.61 \\
		\cite{xu2020layer}ResConv-Prune-B & 73.88  & 76.15  & 2.17 & 13.63 \\
		
		\hline
		\cite{LPSR_zhang}LPSR-40\% & 74.75  & 76.13 & 2.58 & 17.34 \\
		
		\hline
		\cite{DBPwang2019}DBP-A& 74.74  & 76.15 &2.57 & - \\
		\cite{DBPwang2019}DBP-B& 72.44  & 76.15 &2.06 & - \\
		
	\end{tabular}
	\centering
	\caption{Resulting test accuracy, FLOPS and number of parameters of the pruned networks found with our method and our $W$-$B$-$a$-$W$ version of the ResNet50 architecture on the ImageNet dataset in comparison to competing methods for layer pruning.} \label{tab:Imagenet_ResNet50}
\end{table}

\section{Conclusions}\label{sec:conclusions}
Deep neural networks often require excessive computational requirements during training and inference. To address this issue, we have proposed a novel layer pruning algorithm that operates concurrently with the weight learning process. Our algorithm promptly removes nonessential layers during training, thus significantly reducing training time and allowing the training of very deep networks when transfer learning along with pruning methods that operate on fully trained networks are not possible. The proposed algorithm, based on Bayesian variational inference principles, introduces scalar Bernoulli random variables multiplying the outputs of blocks of layers, similar to layer-wise adaptive dropout.
Skip connections around such layer blocks ensure the flow of information through the network and are already present in common residual NN architectures or are introduced as a part of our method.
The optimal network structure is then found as the solution of a stochastic optimization problem over the network weights and parameters describing the variational Bernoulli distributions.
Automatic pruning is effected during the training phase when a variational parameter converges to 0 rendering the corresponding layer permanently inactive. To ensure that the pruning process is carried out in a robust manner with respect to the network's initialization and to prevent the premature pruning of layers, the ``flattening'' hyper-prior \cite{GUENTER2024_robust} is placed on the parameters of the Bernoulli random variables.

A key contribution of this work is a detailed analysis of the stochastic optimization problem showing that all optimal solutions are deterministic networks, i.e., with Bernoulli parameters at either 0 or 1. This property follows as another consequence of adopting the ``flattening'' hyper-prior. In addition, our analysis gives new insight into how the choice of the hyper-prior parameter measures the usefulness of a surviving layer.
Further, we present a stochastic gradient descent algorithm and prove its convergence properties using stochastic approximation results. By considering the implied ODE system, we tie our algorithm to the solutions of the optimization problem. More specifically,
in Theorem~\ref{thm:Alg_conv}, we show that the algorithm converges to the stationary sets of the ODE that coincide with the stationary points of the optimization problem.  Then in Theorem~\ref{thm:conv_min}, we prove that our algorithm indeed converges to a minimum point of it. Lastly in Theorem~\ref{thm:stability}, we examine the dynamics of individual layers and establish that a stationary point with a layer's Bernoulli parameter and weights equal to 0 is an asymptotically stable equilibrium point for that layer independently of the other layers; this result implies that layers that will be eventually pruned, after a point they do so independently of the remaining layers and thus, they can safely removed  during training under a simple condition.

The proposed algorithm is not only computational efficient because of the prompt removal of unnecessary layers but also because it requires very little extra computation during training, consisting of estimating the additional gradients needed to learn the variational parameters. We obtained simple expressions for unbiased estimators of these gradients in terms of the training cost with the corresponding nonlinear layer path active or inactive and also an efficient first order Taylor approximation requiring a single backpropagation pass.
We evaluated the proposed learning/pruning algorithm on the MNIST, CIFAR-10/100 and ImageNet datasets using common LeNet, VGG16 and ResNet architectures and demonstrated that our layer pruning method is competitive
with state-of-the-art methods, both with respect to  achieved pruning ratios and test accuracy.
However, our method is distinguished from other high performing methods since due to its concurrent training and pruning character, it achieves these results with considerably less computation during training.

\appendix

\renewcommand\thesection{\appendixname\ \Alph{section}}
\section{Proof of Theorem~\ref{thm:Alg_conv}}\label{app:assumption_verification}
Theorem~\ref{thm:Alg_conv} follows as a direct application of Theorem~2.1 in \cite[p.127]{kushner2003stochastic} once the assumptions required by the latter are shown to hold.
We have $x(n)\equiv\{W_1^l(n),W_2^l(n),W_3^l(n),\theta^l(n), \, \forall l=1,\dots, L-1\}_{n\geq0}$
and $\nabla_x L(x) = \begin{bmatrix}
\nabla_W L^T & \nabla_\Theta L^T \end{bmatrix}^T$ as well as $\widehat{\nabla_x L(x)} = \begin{bmatrix}
\widehat{\nabla_W L}^T & \widehat{\nabla_\Theta L}^T \end{bmatrix}^T$, where all weights are collectively denoted by $W$ and all Bernoulli parameters by $\Theta$, respectively.
We also define
\begin{align*}
	M(n+1)\equalhat \nabla_x L(x(n))-\widehat{\nabla_x L}(x(n)).
\end{align*}
The following set of assumptions from \cite[p.126]{kushner2003stochastic} are written to fit our notation and are to be verified:
\begin{itemize}
	\item[A2.1] $\sup_n \E\left[\norm{M(n+1)}^2\right] < \infty$,
	\item[A2.2] $\E\left[M(n+1)\right] = \nabla_xL(x) + \beta(n)$, for some RV $\beta(n)$,
	\item[A2.3] $\nabla_xL(\cdot)$ is continuous,
	\item[A2.4] $\sum_{n=0}^\infty a(n) = \infty$ and $\sum_{n=0}^\infty a(n)^2 <\infty$,
	\item[A2.5] $\sum_{n=0}^\infty a(n)\beta(n) < 1$ with probability 1,
	\item[A2.6] $\nabla_xL(\cdot) = f_x(\cdot)$ for a continuously differentiable real-valued $f(\cdot)$ and $f(\cdot)$ is constant on each stationary set.
\end{itemize}
We note that in our application $\beta(n)\equiv0$ as shown during the verification of A2.2 later and
assumptions A2.4 and A2.5 are satisfied for the stepsize $a(n)$ from \eqref{eq:RM}.
Assume  that all activation functions $a^l(\cdot)$ of the network \eqref{eq:struct_res} have derivatives bounded by 1, i.e., $|{a^l}'(\cdot)|\leq 1$ and satisfy $a^l(0)=0$. Typical activation functions such as ReLU and the hyperbolic tangent conform with this assumption.
However, to satisfy Assumption A2.3, we require continuously differentiable activation functions throughout the network. In particular, functions such as the ReLU need to be replaced by similar, smoothed out versions. A2.6 is then also satisfied since $\nabla_x L(x)$ is the gradient of the function $L(x)$, which itself is continuously differentiable, real-valued and constant on each stationary set.

In the following, we verify assumptions A2.1 and A2.2.
Let
\begin{align*}
	 \hat C(W, \hat x_i, \hat y_i,\hat\Xi)\equalhat \frac{N}{B}\sum_{i=1}^B\left[-\log p(\hat y_i\mid \hat x_i,W,\hat\Xi)\right]
\end{align*}
and for simplicity and without loss of generality consider a batch size $B=1$.
Define
\begin{align*}
M_W&\equalhat  \nabla_W L - \widehat{\nabla_W L } = \nabla_W L -\nabla_{W} \hat C(W,\hat x_i, \hat y_i,\hat\Xi) - \lambda W \\
M_{\theta^l}& \equalhat \nabla_\theta L - \widehat{\nabla_\theta L}= \nabla_\theta L - \hat {C_1^l}(W,\hat x_i, \hat y_i,\xi^l=1,\hat{\bar \Xi}^l) + \hat {C_0^l}(W,\hat x_i, \hat y_i,\xi^l=0,\hat{\bar \Xi}^l) + \log\gamma,
\end{align*}
with $\hat C_j^l, \, j=0,1$ as in \eqref{eq:Chats_sampling} and leading to $M = \begin{bmatrix}
M_W^T &  M_\Theta^T
\end{bmatrix}^T$ with $M_\Theta \equalhat [M_{\theta^l}]_l$. In the following, we drop dependence on layer $l$ from $C_j^l, \, j=0,1$ and $\theta^l$ for reason of clarity.
Observe that by \eqref{eq:Cost} we have
\begin{align*}
\E_{\hat\Xi,\hat{\mathcal{D}}}\left[\nabla_{W} \hat C(W,\hat x_i, \hat y_i,\hat\Xi)\right] & =\nabla_{W}\E_{\hat\Xi,\hat{\mathcal{D}}}\left[ \hat C(W,\hat x_i, \hat y_i,\hat\Xi)\right]
= \nabla_{W} C(W,\Theta)\\
\E_{\hat\Xi,\hat{\mathcal{D}}}\left[\hat C_j(W,\hat x_i, \hat y_i,\hat\Xi)\right] &= C_j, \,\, j=1,2.
\end{align*}
Next, because all estimates in $\widehat{\nabla_x L}$ are unbiased single sample estimates, it is $\E_{\hat \Xi,\hat{\cal D}}\left[M\right]=0$, i.e.:
\begin{align*}
\E_{\hat \Xi,\hat{\cal D}}\left[M_W\right] &= \E_{\hat \Xi,\hat{\cal D}}\left[\nabla_{W} C(W,\Theta) + \lambda W - \nabla_{W} \hat C(W,\hat x_i, \hat y_i,\hat\Xi) - \lambda W \right] = 0\\
\E_{\hat \Xi,\hat{\cal D}}\left[M_\theta\right] &= \E_{\hat \Xi,\hat{\cal D}}\left[C_1-C_0 -\log\gamma - \hat C_1 (W,\hat x_i, \hat y_i,\xi=1,\hat{\bar \Xi}) + \hat {C_0}(W,\hat x_i, \hat y_i,\xi=0,\hat{\bar \Xi}) +\log\gamma \right] = 0,
\end{align*}
and hence also $\E_{\hat \Xi,\hat{\cal D}}\left[M_\Theta\right] = 0$, verifying Assumption A2.2 with $\beta(n)\equiv0$ and trivially also A2.5.
We write
\begin{align*}
\norm{M_W}^2 &= \norm{\nabla_{W} C(W,\Theta)-\nabla_{W} \hat C(W,\hat x_i, \hat y_i,\hat\Xi)}^2 \\
&= \norm{\nabla_{W} C(W,\Theta)}^2 -2 (\nabla_{W} \hat C(W,\hat x_i, \hat y_i,\hat\Xi))^T\nabla_{W} C(W,\Theta) +\norm{\nabla_{W} \hat C(W,\hat x_i, \hat y_i,\hat\Xi)}^2,
\end{align*}
and taking its expectation yields
\begin{align*}
\E_{\hat \Xi,\hat{\cal D}}\left[\norm{M_W}^2\right] &= \norm{\nabla_{W} C(W,\Theta)}^2 -2 \E_{\hat \Xi,\hat{\cal D}}\left[\nabla_{W} \hat C(W,\hat x_i, \hat y_i,\hat\Xi) \right]^T\nabla_{W} C(W,\Theta) + \E_{\hat \Xi,\hat{\cal D}}\left[\norm{\nabla_{W} \hat C(W,\hat x_i, \hat y_i,\hat\Xi)}^2\right]\\
&=\E_{\hat \Xi,\hat{\cal D}}\left[\norm{\nabla_{W} \hat C(W,\hat x_i, \hat y_i,\hat\Xi)}^2\right] - \norm{\nabla_{W} C(W,\Theta)}^2\\
&\leq\E_{\hat \Xi, \hat{\mathcal{D}}}\left[\norm{\nabla_{W} \hat C(W,\hat x_i, \hat y_i,\hat\Xi)}^2\right]
+ \norm{\nabla_{W} C(W,\Theta)}^2.
\end{align*}
Using the results of \cite{GUENTER2024_robust}, Appendix B, derived under the same assumptions \eqref{eq:weight_bound} and \eqref{eq:data_moments} as used here, the following bounds are established:
\begin{align*}
	\nabla_{W} C(W,\Theta) \leq 3\eta L\phi_{max},
\end{align*}
where $L$ is the number of hidden layers of the network and the factor of ``$3$'' originates from considering three sets of weights $W_1^l, W_2^l, W_3^l$ in each layer, as well as
\begin{align*}
	E_{\hat \Xi, \hat{\mathcal{D}}}\left[\norm{\nabla_{W} \hat C(W,\hat x_i, \hat y_i,\hat\Xi)}^2\right]  \leq 36 K_D L^2\phi_{max}^2,
\end{align*}
where $K_D<\infty$ is a constant (compare to \cite{GUENTER2024_robust}, Appendix B, eq. B.5-B.6).
Then, we obtain
\begin{align*}
	\E_{\hat \Xi, \hat{\mathcal{D}}}\left[\norm{M_W}^2\right] \leq (36 K_D+9\eta^2)L^2\phi_{max}^2 < \infty
\end{align*}
as required.
Lastly, consider a scalar $M_\theta$ and write
\begin{align*}
|M_\theta| &= |C_1-C_0 - \hat {C_1}(W,\hat x_i, \hat y_i,\xi=1,\hat{\bar \Xi}) + \hat {C_0}(W,\hat x_i, \hat y_i,\xi=0,\hat{\bar \Xi})|\\
&\leq |C_1-C_0| + | \hat {C_1}(W,\hat x_i, \hat y_i,\xi=1,\hat{\bar \Xi}) - \hat {C_0}(W,\hat x_i, \hat y_i,\xi=0,\hat{\bar \Xi})|\\
&\leq \kappa \norm{w} + \kappa \norm{w} \leq 2\kappa \phi_{max}
\end{align*}
using $\norm{w}\leq \phi_{max}$ (see assumption \eqref{eq:weight_bound}), the bound from \eqref{eq:eta_kappa_bnd} and noting that from \cite{GUENTER2024_robust}, Appendix~A, this bound also holds for samples $\hat C_j(\cdot)$ of the function $C_j(\cdot)$, $j=0,1$.
Then, it follows $\norm{M_\theta}^2\leq 4\kappa^2\phi_{max}^2$ and therefore $\E_{\hat \Xi,\hat{\cal D}}\left[\norm{M_\Theta}^2\right]\leq 4L\kappa^2\phi_{max}^2$.
This concludes the verification of Assumption A2.1.



\begin{thebibliography}{10}

\bibitem{girshick_fast_2015}
R.~Girshick, ``Fast r-cnn,'' in {\em International Conference on Computer
  Vision}, pp.~1440--1448, 2015.

\bibitem{noh_learning_2015}
H.~Noh, S.~Hong, and B.~Han, ``Learning deconvolution network for semantic
  segmentation,'' in {\em IEEE International Conference on Computer Vision},
  pp.~1520--1528, 2015.

\bibitem{silver_mastering_2017}
D.~Silver, J.~Schrittwieser, K.~Simonyan, I.~Antonoglou, A.~Huang, A.~Guez,
  T.~Hubert, L.~Baker, M.~Lai, A.~Bolton, Y.~Chen, T.~Lillicrap, F.~Hui,
  L.~Sifre, G.~van~den Driessche, T.~Graepel, and D.~Hassabis, ``Mastering the
  game of {Go} without human knowledge,'' {\em Nature}, vol.~550, no.~7676,
  pp.~354--359, 2017.

\bibitem{ResNet_he_2016}
K.~He, X.~Zhang, S.~Ren, and J.~Sun, ``Deep residual learning for image
  recognition,'' in {\em 2016 IEEE Conference on Computer Vision and Pattern
  Recognition (CVPR)}, pp.~770--778, 2016.

\bibitem{he_identitymappings2016}
K.~He, X.~Zhang, S.~Ren, and J.~Sun, ``Identity mappings in deep residual
  networks,'' in {\em European Conference on Computer Vision -- ECCV 2016},
  pp.~630--645, Springer International Publishing, 2016.

\bibitem{batchnorm_ioffe2015}
S.~Ioffe and C.~Szegedy, ``Batch normalization: Accelerating deep network
  training by reducing internal covariate shift,'' in {\em Proceedings of the
  32nd International Conference on International Conference on Machine Learning
  - Volume 37}, ICML'15, p.~448–456, JMLR.org, 2015.

\bibitem{han_learning_2015}
S.~Han, J.~Pool, J.~Tran, and W.~J. Dally, ``Learning both weights and
  connections for efficient neural networks,'' in {\em International Conference
  on Neural Information Processing Systems - Volume 1}, p.~1135–1143, 2015.

\bibitem{blalock_what_2020}
D.~Blalock, J.~J. Gonzalez~Ortiz, J.~Frankle, and J.~Guttag, ``What is the
  state of neural network pruning?,'' in {\em Machine Learning and Systems},
  vol.~2, pp.~129--146, 2020.

\bibitem{li_pruning_2017}
H.~Li, A.~Kadav, I.~Durdanovic, H.~Samet, and H.~P. Graf, ``Pruning filters for
  efficient convnets,'' {\em arXiv preprint arXiv:1608.08710}, 2016.

\bibitem{he_soft_2018}
Y.~He, G.~Kang, X.~Dong, Y.~Fu, and Y.~Yang, ``Soft filter pruning for
  accelerating deep convolutional neural networks,'' in {\em International
  Joint Conference on Artificial Intelligence}, p.~2234–2240, 2018.

\bibitem{liebenwein_lost_2021}
L.~Liebenwein, C.~Baykal, B.~Carter, D.~Gifford, and D.~Rus, ``Lost in
  {Pruning}: {The} {Effects} of {Pruning} {Neural} {Networks} beyond {Test}
  {Accuracy},'' {\em arXiv:2103.03014 [cs]}, Mar. 2021.
\newblock arXiv: 2103.03014.

\bibitem{SCOP_tang_2020}
Y.~Tang, Y.~Wang, Y.~Xu, D.~Tao, C.~Xu, C.~Xu, and C.~Xu, ``{SCOP}: scientific
  control for reliable neural network pruning,'' in {\em Proceedings of the
  34th {International} {Conference} on {Neural} {Information} {Processing}
  {Systems}}, {NIPS}'20, pp.~10936--10947, Curran Associates Inc., Dec. 2020.

\bibitem{frankle_lottery_2019}
J.~Frankle and M.~Carbin, ``The lottery ticket hypothesis: Finding sparse,
  trainable neural networks,'' in {\em International Conference on Learning
  Representations}, 2019.

\bibitem{Louizos_2017}
C.~Louizos, K.~Ullrich, and M.~Welling, ``Bayesian compression for deep
  learning,'' in {\em Advances in Neural Information Processing Systems}
  (I.~Guyon, U.~V. Luxburg, S.~Bengio, H.~Wallach, R.~Fergus, S.~Vishwanathan,
  and R.~Garnett, eds.), vol.~30, Curran Associates, Inc., 2017.

\bibitem{resnets_ensembles_veit16}
A.~Veit, M.~Wilber, and S.~Belongie, ``Residual networks behave like ensembles
  of relatively shallow networks,'' in {\em Proceedings of the 30th
  International Conference on Neural Information Processing Systems}, NIPS'16,
  (Red Hook, NY, USA), p.~550–558, Curran Associates Inc., 2016.

\bibitem{stoch_depth_huang16}
G.~Huang, Y.~Sun, Z.~Liu, D.~Sedra, and K.~Q. Weinberger, ``Deep networks with
  stochastic depth,'' in {\em Computer Vision -- ECCV 2016} (B.~Leibe,
  J.~Matas, N.~Sebe, and M.~Welling, eds.), (Cham), pp.~646--661, Springer
  International Publishing, 2016.

\bibitem{Srivastava_2014_Dropout}
N.~Srivastava, G.~Hinton, A.~Krizhevsky, I.~Sutskever, and R.~Salakhutdinov,
  ``Dropout: A simple way to prevent neural networks from overfitting,'' {\em
  Journal of Machine Learning Research}, vol.~15, no.~56, pp.~1929--1958, 2014.

\bibitem{chen19_shallowing}
S.~Chen and Q.~Zhao, ``Shallowing deep networks: Layer-wise pruning based on
  feature representations,'' {\em IEEE Transactions on Pattern Analysis and
  Machine Intelligence}, vol.~41, no.~12, pp.~3048--3056, 2019.

\bibitem{DBPwang2019}
W.~Wang, S.~Zhao, M.~Chen, J.~Hu, D.~Cai, and H.~Liu, ``Dbp: Discrimination
  based block-level pruning for deep model acceleration,'' 2019.

\bibitem{LPSR_zhang}
K.~Zhang and G.~Liu, ``Layer pruning for obtaining shallower resnets,'' {\em
  IEEE Signal Processing Letters}, vol.~29, pp.~1172--1176, 2022.

\bibitem{molchanov_pruning_conv}
P.~Molchanov, S.~Tyree, T.~Karras, T.~Aila, and J.~Kautz, ``Pruning
  convolutional neural networks for resource efficient inference,'' in {\em
  International Conference on Learning Representations}, 2017.

\bibitem{sparse_huang2018}
Z.~Huang and N.~Wang, ``Data-driven sparse structure selection for deep neural
  networks,'' in {\em Computer Vision - {ECCV} 2018 - 15th European Conference,
  Munich, Germany, September 8-14, 2018, Proceedings, Part {XVI}} (V.~Ferrari,
  M.~Hebert, C.~Sminchisescu, and Y.~Weiss, eds.), vol.~11220 of {\em Lecture
  Notes in Computer Science}, pp.~317--334, Springer, 2018.

\bibitem{proximal_algs_Parikh2014}
N.~Parikh and S.~Boyd, ``Proximal algorithms,'' {\em Foundations and Trends in
  Optimization}, vol.~1, p.~127–239, Jan. 2014.

\bibitem{xu2020layer}
P.~Xu, J.~Cao, F.~Shang, W.~Sun, and P.~Li, ``Layer pruning via fusible
  residual convolutional block for deep neural networks,'' 2020.

\bibitem{molchanov2017variational}
D.~Molchanov, A.~Ashukha, and D.~Vetrov, ``Variational dropout sparsifies deep
  neural networks,'' in {\em International Conference on Machine Learning},
  pp.~2498--2507, 2017.

\bibitem{Nalisnick_2015}
E.~Nalisnick, A.~Anandkumar, and P.~Smyth, ``A scale mixture perspective of
  multiplicative noise in neural networks,'' {\em arXiv preprint
  arXiv:1506.03208}, 2015.

\bibitem{gal2015dropout}
Y.~Gal and Z.~Ghahramani, ``Dropout as a bayesian approximation: Representing
  model uncertainty in deep learning,'' 2015.

\bibitem{pmlr-v97-nalisnick19a}
E.~Nalisnick, J.~M. Hernandez-Lobato, and P.~Smyth, ``Dropout as a structured
  shrinkage prior,'' in {\em Proceedings of the 36th International Conference
  on Machine Learning} (K.~Chaudhuri and R.~Salakhutdinov, eds.), vol.~97 of
  {\em Proceedings of Machine Learning Research}, pp.~4712--4722, PMLR, 09--15
  Jun 2019.

\bibitem{GUENTER2024_robust}
V.~F.~I. Guenter and A.~Sideris, ``Robust learning of parsimonious deep neural
  networks,'' {\em Neurocomputing}, vol.~566, p.~127011, 2024.

\bibitem{graham2015efficient}
B.~Graham, J.~Reizenstein, and L.~Robinson, ``Efficient batchwise dropout
  training using submatrices,'' {\em arXiv preprint arXiv:1502.02478}, 2015.

\bibitem{cifar10}
A.~Krizhevsky, V.~Nair, and G.~Hinton, ``Cifar-10 (canadian institute for
  advanced research),''

\bibitem{simonyan_deep_2015}
K.~Simonyan and A.~Zisserman, ``Very deep convolutional networks for
  large-scale image recognition,'' in {\em International Conference on Learning
  Representations (oral)}, 2015.

\bibitem{HORNIK1989}
K.~Hornik, M.~Stinchcombe, and H.~White, ``Multilayer feedforward networks are
  universal approximators,'' {\em Neural Networks}, vol.~2, no.~5,
  pp.~359--366, 1989.

\bibitem{bishop_pattern_2006}
C.~M. Bishop, {\em Pattern recognition and machine learning}.
\newblock Information science and statistics, New York: Springer, 2006.

\bibitem{corduneanu_2001_model_sel}
A.~Corduneanu and C.~Bishop, ``Variational bayesian model selection for mixture
  distribution,'' {\em Artificial Intelligence and Statistics}, vol.~18,
  pp.~27--34, 2001.

\bibitem{Lueneberger_Lin_Nonlin_Prog}
D.~G. Luenberger and Y.~Ye, {\em Linear and Nonlinear Programming}.
\newblock Springer New York, 2008.

\bibitem{bengio_estimating_2013}
Y.~Bengio, N.~L{\'e}onard, and A.~Courville, ``Estimating or propagating
  gradients through stochastic neurons for conditional computation,'' {\em
  arXiv preprint arXiv:1308.3432}, 2013.

\bibitem{kushner2003stochastic}
H.~Kushner and G.~Yin, {\em Stochastic Approximation and Recursive Algorithms
  and Applications}.
\newblock Stochastic Modelling and Applied Probability, Springer New York,
  2003.

\bibitem{Jin_escape2017}
C.~Jin, R.~Ge, P.~Netrapalli, S.~M. Kakade, and M.~I. Jordan, ``How to escape
  saddle points efficiently,'' in {\em Proceedings of the 34th International
  Conference on Machine Learning} (D.~Precup and Y.~W. Teh, eds.), vol.~70 of
  {\em Proceedings of Machine Learning Research}, pp.~1724--1732, PMLR, 06--11
  Aug 2017.

\bibitem{Daneshmand1_escape2018}
H.~Daneshmand, J.~Kohler, A.~Lucchi, and T.~Hofmann, ``Escaping saddles with
  stochastic gradients,'' in {\em Proceedings of the 35th International
  Conference on Machine Learning} (J.~Dy and A.~Krause, eds.), vol.~80 of {\em
  Proceedings of Machine Learning Research}, pp.~1155--1164, PMLR, 10--15 Jul
  2018.

\bibitem{lecun-mnisthandwrittendigit-2010}
Y.~LeCun and C.~Cortes, ``{MNIST} handwritten digit database.''
  \url{http://yann.lecun.com/exdb/mnist/}, 2010.

\bibitem{Krizhevsky_09}
A.~Krizhevsky, ``Learning multiple layers of features from tiny images,'' tech.
  rep., 2009.

\bibitem{imagenet_deng2009}
J.~Deng, W.~Dong, R.~Socher, L.-J. Li, K.~Li, and L.~Fei-Fei, ``Imagenet: A
  large-scale hierarchical image database,'' in {\em 2009 IEEE conference on
  computer vision and pattern recognition}, pp.~248--255, Ieee, 2009.

\bibitem{kingma_adam:_2014}
D.~P. Kingma and J.~Ba, ``Adam: A method for stochastic optimization,'' {\em
  International Conference on Learning Representations}, 2015.

\bibitem{Loshchilov2017DecoupledWD}
I.~Loshchilov and F.~Hutter, ``Decoupled weight decay regularization,'' in {\em
  International Conference on Learning Representations}, 2017.

\end{thebibliography}
\end{document}